\documentclass[10pt,twocolumn,letterpaper]{article}

\usepackage{wacv}              % To produce the CAMERA-READY version
\usepackage[ruled,vlined,linesnumbered]{algorithm2e}  % For algorithms
\usepackage{amsmath}
\usepackage{amssymb}
\usepackage{amsfonts}%
\usepackage{mathrsfs}%
\usepackage{multirow}
\usepackage{multicol}
\usepackage{amsthm}
\newtheorem{theorem}{Theorem}
\usepackage[table]{xcolor}

\usepackage{multirow}
\usepackage{graphicx}
\usepackage{booktabs}

\usepackage[accsupp]{axessibility}  % Improves PDF readability for those with disabilities.

\usepackage[table]{xcolor}

\definecolor{ourlightcyan}{RGB}{235, 250, 224}  % light cyan
\definecolor{ourlightred}{RGB}{255, 220, 220}  % light red
\definecolor{myRed}{rgb}{0.808,0.067,0.149}
\definecolor{myGreen}{rgb}{0.067,0.708,0.149}
\definecolor{commentGreen}{rgb}{0,0.5,0.05}
\definecolor{ForestGreen}{RGB}{34,139,34}
\colorlet{TableColor}{ourlightcyan}
\colorlet{MissingColor}{ourlightred}
\newcommand{\ms}[2]{\shortstack{#1\tiny{$\pm$#2}}}
\newcommand{\best}[1]{\textcolor{blue}{#1}}        % First best 
\newcommand{\secondbest}[1]{\textcolor{Orange}{#1}} % Second best 
\definecolor{wacvblue}{rgb}{0.21,0.49,0.74}
\usepackage[pagebackref,breaklinks,colorlinks,allcolors=wacvblue]{hyperref}

\def\wacvPaperID{228} % *** Enter the WACV Paper ID here
\def\confName{WACV}
\def\confYear{2027}

\title{SlotMatch: Distilling Object-Centric Representations for Unsupervised Video Segmentation} 

\author{
    Diana-Nicoleta Grigore\textsuperscript{\rm 1,*},
    Neelu Madan\textsuperscript{\rm 2,*},
    Andreas M{\o}gelmose\textsuperscript{\rm 2},
    Thomas B. Moeslund\textsuperscript{\rm 2},\\
    Radu Tudor Ionescu\textsuperscript{\rm 1,$\diamond$}\\
    \textsuperscript{\rm 1}University of Bucharest, Romania, \quad \textsuperscript{\rm 2}Aalborg University, Denmark\\
    \textsuperscript{\rm *}Equal contribution. \textsuperscript{ $\diamond$}Corresponding author: raducu.ionescu@gmail.com.
\vspace{-0.5cm}}

\begin{document}
\maketitle
\begin{abstract}
Unsupervised video segmentation is a challenging computer vision task, especially due to the lack of supervisory signals coupled with the complexity of visual scenes. To overcome this challenge, state-of-the-art models based on slot attention often have to rely on large and computationally expensive neural architectures. To this end, we propose a simple knowledge distillation framework that effectively transfers object-centric representations to a lightweight student. The proposed framework, called \textsc{SlotMatch}, aligns corresponding teacher and student slots via the cosine similarity, requiring no additional distillation objectives or auxiliary supervision. Unlike other distillation losses, our cosine-based distillation of slots enforces alignment of slot directions (which typically encode relational structure and semantic similarity), while ignoring their norms (which may encode confidence overestimation and dataset-specific scale artifacts). The simplicity of \textsc{SlotMatch} is supported by theoretical and empirical evidence, both indicating that integrating additional losses may be suboptimal. We conduct experiments on three datasets to compare state-of-the-art models, including our \textsc{SlotContrast} teacher, with our distilled student. The results show that our student based on \textsc{SlotMatch} matches and even outperforms its teacher, while using $3.6\times$ fewer parameters and running up to $2.7\times$ faster. Moreover, our student surpasses all other state-of-the-art unsupervised video segmentation models. We release our code at \small{\url{https://github.com/dianagrigore/SlotMatch}}.
\vspace{-0.3cm}
% \keywords{Unsupervised Video Segmentation \and Object-Centric Representation Learning \and Knowledge Distillation}
\end{abstract}

\setlength{\abovedisplayskip}{2.5pt}
\setlength{\belowdisplayskip}{2.5pt}
\setlength{\abovedisplayshortskip}{2.5pt}
\setlength{\belowdisplayshortskip}{2.5pt}

\begin{figure}[!t]
\begin{subfigure}[t]{1.0\linewidth}
  \centering
  % include third image
\includegraphics[width=0.64\linewidth]{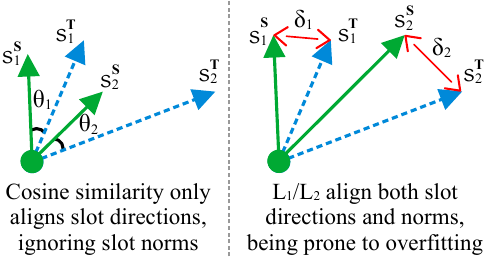}  
\vspace{-0.05cm}
  \caption{Unlike standard L1/L2-based KD, our cosine-based KD reduces overfitting transfer from the teacher by only distilling slot directions. Legend: {\color{Green}green} -- student slots; {\color{RoyalBlue}dashed blue} -- teacher slots.}
  \label{fig:teaserA}
\end{subfigure}
\hfill
\begin{subfigure}[t]{1.0\linewidth}
  \centering
  \vspace{0.1cm}
    \includegraphics[width=0.82\linewidth]{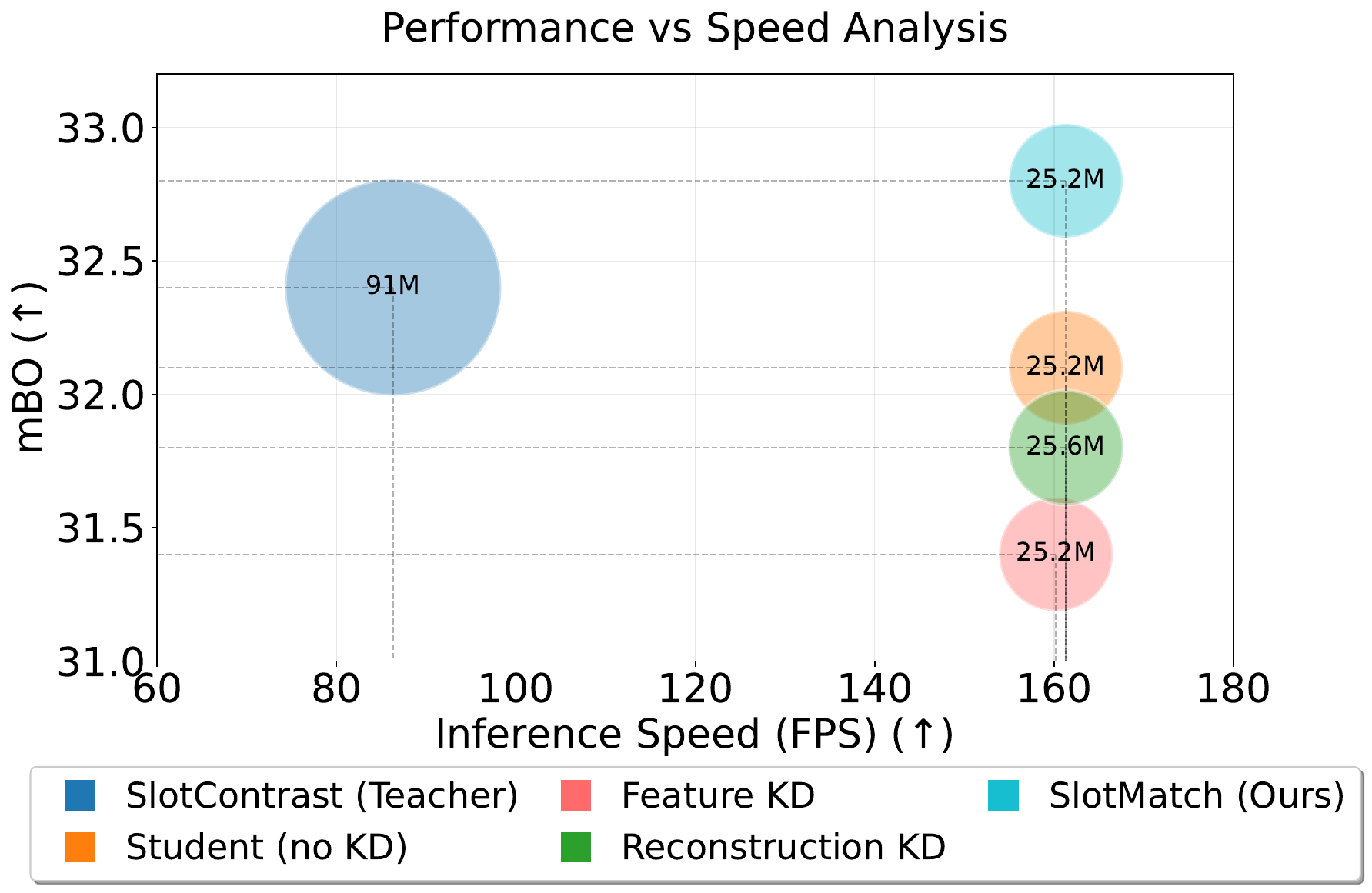}
    \vspace{-0.05cm}
    \caption{Comparison of \textsc{SlotContrast} (teacher) versus various student versions (including our \textsc{SlotMatch}), showing the trade-off between performance (mean Best Overlap or mBO) vs.~inference speed (FPS). Circle area indicates parameter count (in millions). \textsc{SlotMatch} (cyan) outperforms its teacher, while being nearly twice as fast on NVIDIA A100.}
    %\vspace{-0.2cm}
    \label{fig:teaserB}
\end{subfigure}
    \vspace{-0.2cm}
    \caption{Intuition (a) and motivation (b) behind \textsc{SlotMatch}. Unlike alternative KD procedures, \textsc{SlotMatch} boosts both performance and speed, thus bearing a high practical motivation. This is possible due to the main focus of our KD approach based on cosine similarity between slots. \textsc{SlotMatch} aims to transfer only useful information (slot directions), while disregarding harmful dataset-specific artifacts (slot norms). Best viewed in color.}
    \vspace{-0.55cm}
    \label{fig:teaser}
\end{figure}

\vspace{-0.3cm}
\section{Introduction}
\vspace{-0.1cm}

A fundamental goal in machine perception is developing systems that, as humans, understand complex visual scenes as compositions of distinct objects. This capability, studied in the area of \emph{object-centric representation learning} \cite{Burgess-arxiv-2019,Greff-ICML-2019,Locatello-NeurIPS-2020}, is a critical step for building agents that can reason about, interact with, and understand their surroundings. 
Recent advances in self-supervised learning have produced powerful foundational models \cite{caron2021emerging,he2022masked,oquab2023dinov2}. When integrated into slot-based attention frameworks \cite{kipf-ICLR-2022,Locatello-NeurIPS-2020}, these models can discover and segment objects from complex scenes with remarkable fidelity. However, their success is predicated on their scale, as the high computational costs create bottlenecks for deployment. As such, typical self-supervised models are incompatible with resource-constrained environments, where object-centric representations built for video-related tasks would be most valuable.

To overcome this critical trade-off between performance and efficiency, we propose a knowledge distillation strategy tailored for object-centric models. Our method transfers the object discovery capabilities of a large \emph{teacher} model to a lightweight \emph{student}, aiming to create a final model that is not only more suitable for resource-constrained applications, but is also a more effective learner.

Our core contribution is a systematic study of how to distill object-centric representations. We investigate different targets and objectives and find that aligning slot directions provides a particularly simple and effective solution. Objects in the scene are disentangled through slots (vectors) pointing in different directions in the semantic embedding space. We therefore frame the primary distillation objective as a slot direction matching loss, termed \textsc{SlotMatch}, where a slot from the student model serves as the ``positive'' and the corresponding teacher slot provides the fixed (target) ``anchor''. Crucially, we propose to distill slots via cosine similarity, which only aims to transfer the slot directions from the teacher to the student model, as shown in Figure \ref{fig:teaserA}. This is fundamentally different from standard knowledge distillation (KD) objectives, such as $L_1$ or $L_2$, where both slot directions and slot norms are transferred. Our approach is consistent with the study of Zhang \etal~\cite{Zhang-AAAI-2025}, who argue that standard KD transfers both good and wrong (bias) knowledge from the teacher. They propose a KD framework to explicitly rectify bias transfer, obtaining students that surpass their teachers. Similarly, our KD based on cosine similarity loss enforces alignment of slot directions (good knowledge), ignoring their norms (wrong knowledge). This is based on the fact that feature (slot) magnitude can encode confidence overestimation and dataset-specific scale artifacts, while feature (slot) direction encodes relational structure and semantic similarity. Hence, \textsc{SlotMatch} can achieve a regularization effect, enabling the student model to even surpass its teacher (see Figure \ref{fig:teaserB}). We go beyond simple intuition to confirm the validity of performing KD via our \textsc{SlotMatch}. Indeed, we provide theoretical and empirical evidence suggesting that our simple distillation objective is sufficient (additional distillation objectives can degrade performance) and effective (\textsc{SlotMatch} matches or outperforms its teacher).

We carry out experiments on three benchmark datasets for unsupervised video segmentation, MOVi-E \cite{ghorbani2021movi}, YTVIS-2021 \cite{vis2021} and DAVIS 2017 \cite{Perazzi-CVPR-2016}. We compare our student based on \textsc{SlotMatch} with its state-of-the-art teacher, \textsc{SlotContrast} \cite{Manasyan-CVPR-2025}, as well as other competitive methods from recent literature \cite{Aydemir-NeurIPS-2023, kipf-ICLR-2022, Singh-NeurIPS-2022,Zadaianchuk-NeurIPS-2023, Seitzer-arxiv-2022, kara2024diod}. The results indicate that our student outperforms state-of-the-art models on all datasets. Furthermore, our student contains $3.6\times$ fewer parameters and runs up to $2.7\times$ faster than its teacher model, \textsc{SlotContrast} (see Figure \ref{fig:teaserB} and Table \ref{table:ablation_timings}). We also confirm that \textsc{SlotMatch} can be applied to various teacher-student pairs, achieving the same effect, that of maintaining performance, while reducing inference time.

In summary, our contribution is threefold:
\begin{itemize}
    \item We systematically study distillation for object-centric representations and introduce \textsc{SlotMatch}, a knowledge distillation framework to distill slot attention by minimizing the cosine similarity between corresponding teacher and student slots, as a simple and effective strategy.
    \item We provide theoretical evidence indicating that cosine slot alignment can already control reconstruction discrepancy without requiring student and teacher norms to match, while reconstruction distillation introduces an additional pressure that may erase this scale freedom.
    \item We empirically show that our student based on \textsc{SlotMatch} is both effective and efficient, matching or surpassing its teachers in terms of multiple performance metrics, while being $1.9\times$ to $2.7\times$ faster and $3.6\times$ smaller. 
\end{itemize}

\vspace{-0.2cm}
\section{Related Work}
\vspace{-0.1cm}

\noindent \textbf{Object-centric representation learning.}
Unsupervised scene decomposition aims to discover object-based structure in raw perceptual inputs, without supervision. Early methods approached this task via perceptual grouping \cite{Greff-NeurIPS-2016}, spatial mixture models \cite{Greff-NeurIPS-2017}, or foreground-background separation \cite{Yang-ICCV-2021}. Recent approaches focused on \emph{slot attention} \cite{Locatello-NeurIPS-2020}, which uses iterative attention to bind latent learnable vectors, called \emph{slots}, to individual objects. Since its introduction, the paradigm has been extended to incorporate real-world images \cite{Seitzer-arxiv-2022}, auto-regressive decoders with patch-order permutation \cite{kakogeorgiou2024spot}, adaptive slot counts \cite{Fan-CVPR-2024}, learned-query initialization \cite{Jia-ICLR-2023}, language-controlled slots \cite{didolkar2025ctrl} and scale-invariant pipelines \cite{Biza-ICML-2023}. While all of these have shown strong results in image-based object discovery, their adaptation to video requires additional mechanisms to ensure temporal consistency.

\noindent \textbf{Object-centric video models.} Slot attention was originally extended to video through methods such as SAVi \cite{kipf-ICLR-2022} and SAVi++ \cite{Elsayed-NeurIPS-2022}, which introduced cross-frame attention and optical flow to promote temporal coherence. Later, STEVE \cite{Singh-NeurIPS-2022} added sequential latent dynamics for video generation. VideoSAUR \cite{Zadaianchuk-NeurIPS-2023} and SlotContrast \cite{Manasyan-CVPR-2025} further enhanced the approach with contrastive objectives and stronger backbones, \eg~DINOv2 \cite{oquab2023dinov2}. While effective, these methods typically rely on computationally-intensive encoders, limiting their applicability to real-time or embedded scenarios. Our work addresses this limitation by transferring object-centric knowledge from a high-capacity video model to a lightweight model via distillation.

\noindent \textbf{Knowledge distillation.} Knowledge distillation (KD) transfers information from a large ``teacher'' model to a compact ``student'' by aligning output distributions \cite{Hinton-arxiv-2015}, intermediate activations \cite{Romero-arxiv-2014}, or feature Jacobians \cite{Czarnecki-NeurIPS-2017}, among others. Most of the existing KD studies are focused on traditional tasks (\eg~classification \cite{iordache2025multi}, language modeling \cite{guminillm}), with limited application to object-centric learning. 
To the best of our knowledge, there are only a few recent papers that explore distillation from scene decomposition \cite{kakogeorgiou2024spot, kara2024diod, li2024object, liao2025forla, Seitzer-arxiv-2022}. DIOD \cite{kara2024diod} introduced a self-distillation strategy in video slot attention, where an EMA teacher supervises the student by refining slot masks using optical flow and static cues. DINOSAUR \cite{Seitzer-arxiv-2022} is a framework that distills high-level features from a pre-trained DINOv2 encoder by training a slot-based model to reconstruct them, enabling slot emergence on real-world images. Some works \cite{li2024object} used slot attention to distill object-centric features from an image-based detector to an event-based one via slot-aware cross-modal alignment, while others proposed federated slot learning \cite{liao2025forla}, using a teacher-student decoder setup to distill shared object-centric representations across distributed domains. 

Nevertheless, none of the aforementioned studies considers structured slot-based representations. Our approach departs from standard distillation by aligning latent object slots directly, using a simple and effective cosine-based objective in the slot space. To our knowledge, \textsc{SlotMatch} is the first method to perform distillation at the level of semantic slot representations in video.

\vspace{-0.2cm}
\section{Method}
\vspace{-0.1cm}

We propose \textsc{SlotMatch}, a distillation framework designed to transfer temporally-consistent object-centric representations from a large teacher model to a compact student model. Our core insight is to operate directly in the slot space. As such, we introduce a direct similarity-based distillation objective that aligns semantic slot representations between teacher and student models. Unlike standard methods, \textsc{SlotMatch} avoids conventional pixel-wise or feature-level distillation, instead relying on a direct semantic match between corresponding slots. This is achieved by focusing on distilling slot directions, while refraining from transferring slot norms. 

\noindent \textbf{Problem setup.}
Let $\mathbf{T}$ denote a frozen pre-trained teacher model with a high-capacity encoder, and $\mathbf{S}$ denote a lightweight trainable student. Both operate on video frames and produce a fixed number of slot representations per frame, capturing object-centric semantics. Specifically, for a video frame $\boldsymbol{x}_t$, the teacher and the student produce the slot representations defined below:
\begin{equation}
\boldsymbol{S}^{\mathbf{T}}_t\!=\!\{s^{\mathbf{T}}_{t,n}\}_{n=1}^N\!\in\! \mathbb{R}^{N \times d},\, \boldsymbol{S}^{\mathbf{S}}_t\!=\!\{s^{\mathbf{S}}_{t,n}\}_{n=1}^N\!\in\! \mathbb{R}^{N \times d},
\end{equation}
where each slot $s_{t,n} \in \mathbb{R}^{d}$ encodes information about an object in the scene. We aim to train the student model such that its slots match the semantics of the teacher, preserve temporal consistency, and support scene reconstruction. 

\noindent \textbf{\textsc{SlotMatch} framework.}
In Figure~\ref{fig_pipeline}, we showcase our \textsc{SlotMatch} framework, which operates in a dual-model fashion, with the teacher and student models processing identical video inputs in parallel, through a slot-attention-based encoder-decoder architecture. The application of \textsc{SlotMatch} is hereby exemplified on top of \textsc{SlotContrast} \cite{Manasyan-CVPR-2025}, a state-of-the-art teacher model. Yet, \textsc{SlotMatch} is compatible with other teachers, as shown in our experiments (see the supplementary material).

\begin{figure*}[t]
  \centering
  \includegraphics[width=0.67\linewidth]{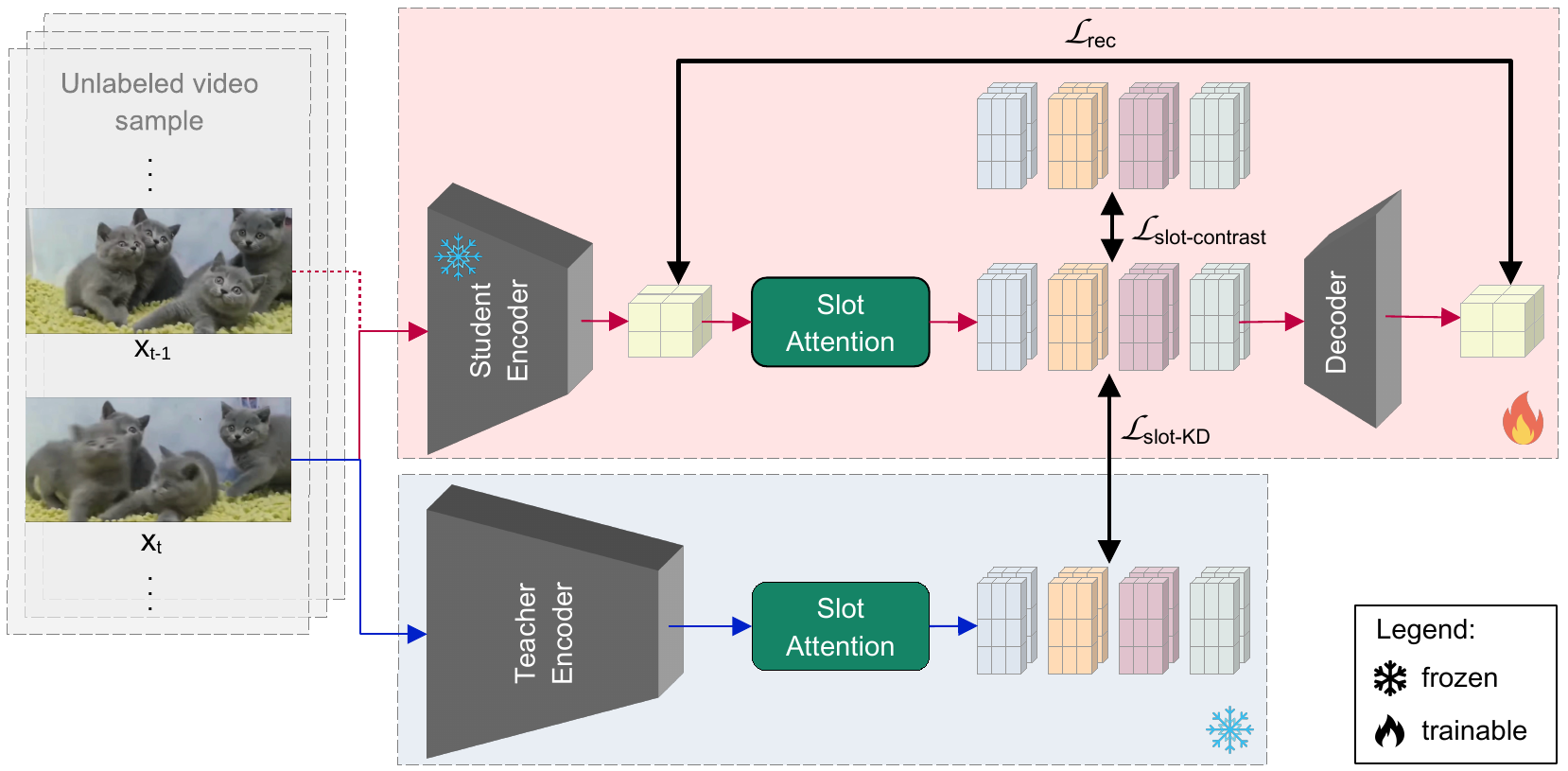}
      \vspace{-0.3cm}
  \caption{Our \textsc{SlotMatch} framework performs knowledge distillation from a large frozen teacher model, \eg \textsc{SlotContrast} \cite{Manasyan-CVPR-2025}, to a compact trainable student model. Both models process video frames through slot attention mechanisms, with the student learning through three loss components: reconstruction ($\mathcal{L}_{\text{rec}}$), temporal consistency ($\mathcal{L}_{\text{slot-contrast}}$), and our angular slot matching loss ($\mathcal{L}_{\text{slot-KD}}$) that directly aligns corresponding slot directions between teacher and student models using cosine similarity.}
  \vspace{-0.45cm}
\label{fig_pipeline} 
\end{figure*}

First, a pre-trained backbone, \eg DINOv2 \cite{oquab2023dinov2}, extracts spatial feature maps from each video frame. The teacher employs a higher-capacity encoder (\eg ViT-B), while the student uses a lighter variant (\eg ViT-S). Both encoders are frozen, but they contain a shallow trainable MLP that projects encoder features into a joint $d$-dimensional slot space. For each frame, we initialize a fixed number of $N$ slots and iteratively update them via attention applied to encoded features. Slot attention compresses scene information into object-centric latent vectors. In the final part, the student employs a decoder that reconstructs the original features from slots. This ensures that slots are informative and can guide unsupervised segmentation. The teacher also uses a similar decoder during its own training phase, but we do not employ it in the distillation stage, since \textsc{SlotMatch} operates only on slots, not on decoder outputs. 

% To avoid clutter, we refrain from depicting the decoder of the teacher in Figure~\ref{fig_pipeline}.

\noindent \textbf{Slot-level distillation.}
In the teacher-student training phase, we consider two options to determine the correspondence among teacher and student slots. One option is to perform explicit matching, \eg~via Hungarian assignment, while the other is to assume slot index correspondence, \ie~the $n$-th slot of the student corresponds to the $n$-th slot of the teacher. Both options are perfectly reasonable. The first option (Hungarian matching) aims to capitalize on any prior alignment between teacher and student slots, while the second option assumes that slot correspondence can be achieved entirely through learning. We ablate slot correspondence strategies in Table \ref{table:hungarian_matching} and find negligible difference between Hungarian and direct assignment setups. We thus adopt the simple direct assignment in our framework.

Each model produces $N$ slot representations of dimension $d$, each of them encoding semantic information about a distinct object in the scene. The teacher model, through its larger encoder, learns robust object representations via slot directions, capturing fine-grained semantic details, but it may also overfit slot norms to dataset biases. In our framework, we aim to distill slot directions into the student to transfer object disentanglement capabilities from the teacher, while preventing bias transfer (encoded into slot norms). To achieve this goal, our core contribution is to employ a cosine-based slot distillation loss, denoted as $\mathcal{L}_{\text{slot-KD}}$, that aligns each student slot with its teacher counterpart:
\begin{equation}\label{eq_slot_KD}
\!\!\mathcal{L}_{\text{slot-KD}}\!=\!\frac{1}{B\!\cdot\!T\!\cdot\!N} \sum_{b=1}^{B}\!\sum_{t=1}^{T}\!\sum_{n=1}^{N}\!\left(\!1\!-\! \frac{\langle s^{\mathbf{S}}_{b,t,n}, s^{\mathbf{T}}_{b,t,n}\rangle}{\|s^{\mathbf{S}}_{b,t,n}\|\!\cdot\! \|s^{\mathbf{T}}_{b,t,n}\|}\!\right)\!,
\end{equation}
where $B$ is the mini-batch size, $T$ is the number of frames in a video, and $N$ is the number of slots. This formulation directly distills the object-centric latent space, promoting structured and efficient knowledge transfer. Each student slot serves as a ``positive'' that must align to its corresponding teacher slot through cosine similarity-based optimization. The target teacher slot serves a fixed ``anchor'', since the teacher is completely frozen during distillation. Our formulation creates attractive forces that guide the student slots to match the directions of the teacher slots, enabling the transfer of useful semantic information.

\noindent \textbf{Temporal and reconstruction losses.}
Following Manasyan et al.~\cite{Manasyan-CVPR-2025}, we integrate two auxiliary objectives to improve slot quality for the student, namely the slot-slot contrastive loss ($\mathcal{L}_{\text{slot-contrast}}$), and the feature reconstruction loss ($\mathcal{L}_{\text{rec}}$):
\begin{equation}\label{eq_slot_contrast}
\begin{split}
\mathcal{L}&_{\text{slot-contrast}}\!=\! \frac{1}{B\!\cdot\!T\!\cdot\!N}\!\sum_{b=1}^{B}\! \sum_{t=1}^{T-1}\!\sum_{n=1}^{N} 
        \\
        & -\log \frac{
        \exp \left( \text{sim}(s^{\mathbf{S}}_{b,t,n},\, s^{\mathbf{S}}_{b,t+1,n})/\tau \right)}{
        \sum_{b',t',n'} \exp \left( \text{sim}(s^{\mathbf{S}}_{b,t,n},\, s^{\mathbf{S}}_{b',t',n'})/\tau \right)},
\end{split}
\end{equation}
\begin{equation}\label{eq_reconst}
\mathcal{L}_{\text{rec}} = \frac{1}{B\cdot T} \sum_{b=1}^{B} \sum_{t=1}^{T} \left\|\hat{\boldsymbol{F}}{}^{\mathbf{S}}_{b,t} - \boldsymbol{F}{}^{\mathbf{S}}_{b,t}\right\|^2,
\end{equation}
where $\boldsymbol{F}{}^{\mathbf{S}}_{b,t}$ and $\hat{\boldsymbol{F}}{}^{\mathbf{S}}_{b,t}$ represent target and generated reconstructions, respectively. This decomposition assigns distinct responsibilities to each loss component. Critically, the absence of explicit negative sampling in the distillation loss is compensated by the concurrent optimization of $\mathcal{L}_{\text{slot-contrast}}$, which provides the necessary repulsive forces to maintain slot distinctiveness for the student. While the slot distillation loss enforces semantic alignment between corresponding teacher-student slot pairs, the slot-slot contrastive loss maintains representational diversity among student slots, ensures temporal consistency by attracting slots representing the same object across consecutive frames, while also repelling slots from different objects within the mini-batch. The feature reconstruction loss $\mathcal{L}_{\text{rec}}$ ensures that student slots contain sufficient information to reconstruct features given by the student encoder, via its trainable decoder.

\noindent \textbf{Overall training objective.}
Training proceeds as formally described in Algorithm 1 from the supplementary. For each mini-batch of $B$ videos with $T$ frames each, we extract per-frame slot representations $\boldsymbol{S}^{\mathbf{T}}_{b,t}$ and $\boldsymbol{S}^{\mathbf{S}}_{b,t}$ from the teacher and student models, where $b$ is the video index and $t$ is the frame index. Then, a decoder reconstructs the input features from the student slots, producing $\hat{\boldsymbol{F}}{}^{\mathbf{S}}_{b,t}$. We next compute three loss terms: (i) a reconstruction loss $\mathcal{L}_{\text{rec}}$ to ensure that student slots retain sufficient scene-level information; (ii) a temporal contrastive loss $\mathcal{L}_{\text{slot-contrast}}$ to promote both consistency and diversity of slots across adjacent frames; and (iii) our core distillation objective $\mathcal{L}_{\text{slot-KD}}$ to align each student slot with its teacher counterpart via cosine similarity. These are combined into a weighted loss used to update the student parameters via gradient descent, while the teacher remains frozen during the whole process. The complete training objective is defined as: 
\begin{equation}
\mathcal{L}_{\text{total}} = \mathcal{L}_{\text{rec}} + \alpha \cdot \mathcal{L}_{\text{slot-contrast}} + \beta \cdot \mathcal{L}_{\text{slot-KD}},
\end{equation}
where $\alpha$ and $\beta$ control the relative importance of temporal consistency and knowledge transfer, respectively. 
% Algorithm ~\ref{alg:slot_distillation} summarizes our complete SlotMatch training procedure. 
The key advantages of our slot-based approach are its simplicity and efficiency. Unlike methods requiring complex assignment algorithms or feature-level matching, our direct correspondence assumption eliminates computational overhead, while maintaining effective knowledge transfer. 

\noindent \textbf{Theoretical justification.}
We conjecture that it is sufficient to employ the loss in Eq.~\eqref{eq_slot_KD} to perform effective knowledge distillation, without forcing teacher and student slot norms to match. More precisely, our cosine-based slot alignment can directly control feature reconstruction discrepancy without requiring student and teacher norms to match, while reconstruction KD introduces an additional pressure that may erase this scale freedom. To support our conjecture, which simplifies the distillation framework, we rely on the following theorem:
\begin{theorem}\label{prop_slot_distill}
Let $s^\mathbf{T}\!\in \mathbb{R}^d$ be a teacher slot and $s^\mathbf{S}\!\in \mathbb{R}^d$ a student slot, with $\left\|s^\mathbf{T}\right\| > 0$ and $\left\|s^\mathbf{S}\right\| > 0$. Let $f\!: \mathbb{R}^d \to \mathbb{R}^m$ be a $K_f$-Lipschitz neural network that decodes the slots into features. If the slot distillation loss converges to a constant $c$, \ie:
\begin{equation}\label{eq_slot_distill}
\mathcal{L}_{\text{slot-KD}}\left(s^\mathbf{T},s^\mathbf{S}\right)=1 - \frac{\langle s^\mathbf{T}, s^\mathbf{S} \rangle}{\|s^\mathbf{T}\|\cdot\|s^\mathbf{S}\|} = c,
\end{equation}
then:
\begin{equation}\label{eq_feat_distill}
\begin{split}
\mathcal{L}_{\text{rec-KD}}&\!\left(f\!\left(s^\mathbf{T}\right)\!,f\!\left(s^\mathbf{S}\right)\!\right)\!=\! \left\|f\!\left(s^\mathbf{T}\right)\!-\!f\!\left(s^\mathbf{S}\right)\!\right\|^2\!\\
&\leq\!K_f^2\!\cdot\!\left[\!\left(\|s^\mathbf{T}\|\!-\!\|s^\mathbf{S}\|\right)^2\!+\!2\!\cdot\!\|s^\mathbf{T}\|\!\cdot\!\|s^\mathbf{S}\|\!\cdot\!c \right]\!.
\end{split}
\end{equation}
\end{theorem}
\begin{proof}
The proof is given in the supplementary.
\end{proof}
Remarks:
\begin{itemize}
    \item If the slots are normalized (\ie~$\|s^\mathbf{T}\|\!=\!\|s^\mathbf{S}\|\!=\!1$), the bound simplifies to:
\begin{equation}\label{eq_bound_2}
\mathcal{L}_{\text{rec-KD}}\!\left(f\!\left(s^\mathbf{T}\right)\!,f\!\left(s^\mathbf{S}\right)\!\right)\leq 2 \cdot K_f^2 \cdot c.
\end{equation}
    \item If $c\!\rightarrow\!0$, the teacher-student reconstruction loss $\mathcal{L}_{\text{rec-KD}}$ is only bounded by the difference between slot norms:
\begin{equation}\label{eq_bound_3}
\mathcal{L}_{\text{rec-KD}}\!\left(f\!\left(s^\mathbf{T}\right)\!,f\!\left(s^\mathbf{S}\right)\!\right)\leq K_f^2 \cdot \left(\|s^\mathbf{T}\|\!-\!\|s^\mathbf{S}\|\right)^2.
\end{equation}
\end{itemize}

% Therefore, once the directions are aligned, the remaining teacher-student reconstruction discrepancy can be dominated by their norm difference. Explicitly minimizing $\mathcal{L}_{\text{rec-KD}}$ then pressures the student to reproduce aspects of the teacher that cosine distillation intentionally ignores, including teacher-specific scale. In this sense, reconstruction KD can negate the principal advantage of cosine distillation.

%If the smaller student norm is a benign consequence of lower capacity, useful regularization, calibration, or a different internal parameterization, reconstruction KD may force the student toward an unnecessary or harmful teacher scale. 

Theorem \ref{prop_slot_distill} decomposes the reconstruction discrepancy into a component induced by angular slot misalignment and a component induced by the teacher-student norm gap. Our cosine objective controls the former (angular discrepancy), while leaving the latter (norm discrepancy) unconstrained. Consequently, a student can achieve accurate semantic alignment even when its slot norms differ from those of the teacher. When angular error is small, adding reconstruction KD primarily introduces pressure to reduce the residual scale-dependent discrepancy. This may force the compact student to imitate teacher-specific magnitudes and decoded details that are not required for downstream segmentation. The theorem therefore establishes the redundancy and potential source of conflict between $\mathcal{L}_{\text{slot-KD}}$ and $\mathcal{L}_{\text{rec-KD}}$. Nevertheless, we go beyond theoretical evidence, and consider the possibility of introducing the teacher-student reconstruction loss into the optimization objective.
We empirically test this objective and find that it does not help convergence to a better optimum. Hence, both theoretical and empirical evidence support our conjecture, suggesting that our simple distillation objective is sufficient.

\vspace{-0.15cm}
\section{Experiments}
\vspace{-0.1cm}
\subsection{Datasets}
\vspace{-0.1cm}

We evaluate \textsc{SlotMatch} on both synthetic and real-world video datasets to assess its effectiveness across controlled, real-world and zero-shot scenarios.

\noindent \textbf{MOVi-E.} The Multi-Object Video (MOVi-E) dataset \cite{ghorbani2021movi} is generated using the Kubric simulator, which provides ground-truth object segmentation for precise evaluation. MOVi-E includes up to 23 objects per scene and is filmed with linear camera motion. The dataset includes complex object interactions, occlusions, and realistic textures. The videos have a 24-frame length and a 256$\times$256 resolution. We use the official train and validation splits.  

\noindent \textbf{YTVIS-2021.}
To evaluate scalability to real-world scenarios, we use the YouTube Video Instance Segmentation 2021 (YTVIS-2021) dataset \cite{vis2021}. It contains unconstrained real-world videos, capturing diverse scenes including indoor/outdoor environments, multiple object categories, and complex interactions. Videos vary in length (up to 76 frames) and resolution. We resize them to 518$\times$518 pixels. We report results on the official split.

\noindent \textbf{DAVIS 2017.}
We report results on the DAVIS 2017 dataset \cite{Perazzi-CVPR-2016} to evaluate fine-grained segmentation quality. This benchmark provides videos with dense temporal annotations, with every frame labeled using pixel-accurate object masks. The videos span diverse real-world scenarios, including people, animals, vehicles, and sports activities, capturing challenging phenomena like fast motion and cluttered backgrounds. 

\noindent \textbf{YTVIS-2021$\rightarrow$OVIS.}
We conduct zero-shot experiments on Occluded Video Instance Segmentation (OVIS) \cite{qi2022occluded}, a dataset specifically focused on challenging scenarios. OVIS contains real-world videos with heavy occlusions, making it particularly suitable for testing temporal consistency when objects frequently disappear and reappear. The dataset features 607 training videos and 140 validation videos across 25 object categories.

\begin{table*}[th]
\centering
\setlength{\tabcolsep}{1.8pt}
\resizebox{1.0\linewidth}{!}{%
\begin{tabular}{|l|cc|cc|cc|cc|cc|cc|}
\hline
\multirow{3}{*}{Method}  & \multicolumn{4}{c|}{MOVi-E} & \multicolumn{4}{c|}{YTVIS-2021} & \multicolumn{4}{c|}{DAVIS 2017} \\ 
\cline{2-13}
& \multicolumn{2}{c|}{Image} & \multicolumn{2}{c|}{Video} & \multicolumn{2}{c|}{Image} & \multicolumn{2}{c|}{Video}  & \multicolumn{2}{c|}{Image} & \multicolumn{2}{c|}{Video} \\ 
\cline{2-13}
& \rotatebox{0}{FG-ARI $\uparrow\,$} & \rotatebox{0}{mBO $\uparrow$} & \rotatebox{0}{FG-ARI $\uparrow$} & \rotatebox{0}{mBO $\uparrow$} & \rotatebox{0}{FG-ARI $\uparrow$} & \rotatebox{0}{mBO $\uparrow$} & \rotatebox{0}{FG-ARI $\uparrow$} & \rotatebox{0}{mBO $\uparrow$} & \rotatebox{0}{FG-ARI $\uparrow$} & \rotatebox{0}{mBO $\uparrow$} & \rotatebox{0}{FG-ARI $\uparrow$} & \rotatebox{0}{mBO $\uparrow$}  \\ %
\hline
\hline
SAVi \cite{kipf-ICLR-2022} & 39.2 & - & 42.8 & 16.0 & - & - & - & - & - & - & - & - \\ %
STEVE \cite{Singh-NeurIPS-2022} & 54.1 & -  & 50.6 & 26.6 & - & -  & 15.0 & 19.1 & - & - & 3.7 & 5.9 \\ %
VideoSAUR \cite{Zadaianchuk-NeurIPS-2023} & 78.4 & - & 73.9 & \best{\textbf{35.6}} & - & - & 28.9 & 26.3 & 12.6 & 23.9 & 8.3 & 18.4  \\ %
VideoSAURv2 \cite{Manasyan-CVPR-2025} & - & - & 77.1 & \secondbest{\textbf{34.4}} & - & - & 31.2 & 29.7  & 11.8 & 23.0 & 7.6 & \secondbest{\textbf{19.2}} \\ %
SOLV \cite{Aydemir-NeurIPS-2023} & 80.8 & - & - & - & 29.1 & - & - & -  & 32.2 & - & - & - \\ %
DINOv2 \cite{caron2021emerging} + Clustering & - & - & - & - & 16.2 & 28.2 & - & - & 24.2 & 14.4 & - & - \\
\hline
DIAS \cite{zhao2025slot} & - & - & - & - & \best{\textbf{52.1}} & 33.2 & - & - & 64.7 & \best{\textbf{36.8}} & - & - \\ %
DINOSAUR \cite{Seitzer-arxiv-2022} & 65.1 & 29.1 & - & - & - & - & - & - & - & - & - & - \\
DIOD \cite{kara2024diod} & 82.2 & - & - & - & - & - & - & - & - & - & - & - \\ %
\hline
\textsc{SlotContrast} \cite{Manasyan-CVPR-2025} (teacher)
& \secondbest{\textbf{83.9}} & \secondbest{\textbf{32.4}} & \secondbest{\textbf{81.7}} & 28.6
& 45.5 & \secondbest{\textbf{39.7}} & 36.8 & \secondbest{\textbf{32.4}}
&  89.1 & 12.1 & 73.3 & 11.8   \\ 
\rowcolor{TableColor} {\textsc{SlotMatch} (ours)} 
& \best{\textbf{\ms{84.1}{0.3}}} 
& \best{\textbf{\ms{33.6}{0.8}}} 
& \best{\textbf{\ms{81.8}{0.6}}} 
& \ms{30.5}{0.8} 
& \secondbest{\textbf{\ms{45.8}{0.4}}} 
& \best{\textbf{\ms{39.8}{0.3}}} 
& \best{\textbf{\ms{37.3}{0.1}}} 
& \best{\textbf{\ms{32.8}{0.3}}} 
& \best{\textbf{\ms{95.6}{0.5}}} 
& \secondbest{\textbf{\ms{30.6}{4.0}}} 
& \best{\textbf{\ms{92.8}{1.6}}} 
& \best{\textbf{\ms{31.6}{5.3}}} \\
\hline
\end{tabular}
}
\vspace{-0.3cm}
\caption{State-of-the-art comparison on MOVi-E, YTVIS-2021, and DAVIS 2017. We report FG-ARI and mBO at both image and video levels.  The top two scores for each metric are highlighted in \best{\textbf{blue bold}} (best) and \secondbest{\textbf{orange bold}} (second), respectively. For \textsc{SlotMatch}, we report average scores and standard deviations over 3 runs.}
\label{table:sota_main}
\vspace{-0.45cm}
\end{table*}

\vspace{-0.1cm}
\subsection{Evaluation Metrics}
\vspace{-0.1cm}

We use two complementary evaluation metrics, namely the Foreground Adjusted Rand Index (FG-ARI) and the mean Best Overlap (mBO), and apply them at both image and video levels. \textbf{FG-ARI} measures how well the model groups pixels belonging to the same object. With a range of 0 to 1, and a higher value indicating better performance, its focus is on \emph{object discovery quality}, \ie~if the model correctly identifies which pixels belong to the same object. FG-ARI compares predicted object masks against ground-truth segmentation masks. \textbf{mBO} measures the spatial precision of object masks, \ie~how accurately the predicted masks align with ground-truth boundaries. For each predicted mask, it finds the ground-truth mask with the highest IoU, then averages the IoU across all objects. mBO focuses on \emph{segmentation mask quality}, measuring how precise and well-defined the predicted object boundaries are. 

For the image-level evaluation, we compute the FG-ARI and mBO metrics per frame, then average them. For video, we compute the metrics across entire video sequences, accounting for temporal consistency.

\vspace{-0.1cm}
\subsection{Baselines and Hyperparameters}
\vspace{-0.1cm}

\noindent
\textbf{State-of-the-art competitors.}
We compare \textsc{SlotMatch} against a range of recent object-centric video segmentation models. These include slot-based methods such as SAVi~\cite{kipf-ICLR-2022}, STEVE~\cite{Singh-NeurIPS-2022}, SOLV~\cite{Aydemir-NeurIPS-2023}, VideoSAUR~\cite{Zadaianchuk-NeurIPS-2023}, and its DINOv2-enhanced variant, VideoSAURv2~\cite{Manasyan-CVPR-2025}. We reproduce the results of \textsc{SlotContrast}~\cite{Manasyan-CVPR-2025}, which serves as the teacher in our framework based on knowledge distillation. In addition, we compare with a number of knowledge distillation methods, namely DIOD \cite{kara2024diod}, DINOSAUR \cite{Seitzer-arxiv-2022}, and DIAS \cite{zhao2025slot}.

\noindent
\textbf{Baseline and ablated KD strategies.}
We also compare with several KD baselines, each employing a different distillation strategy, as follows:
\begin{itemize}
    \item \textbf{Student (no KD)} is a \textsc{SlotContrast} model trained from scratch, using a smaller encoder version from the same family, namely DINOv2-small instead of DINOv2-base.
    \item \textbf{Feature KD} is a student model (based on DINOv2-small) which distills features from the frozen DINOv2-base teacher encoder. These are passed through a two-layer MLP, and the student is trained to match the transformed features via $L_2$ loss.
    \item \textbf{Reconstruction KD} distills the reconstructed output of the teacher model to the student model via $\mathcal{L}_{\text{rec-KD}}$, essentially minimizing the $L_2$ loss between the two outputs.
    \item \textbf{Contrastive KD} aligns teacher and student slot representations via contrastive learning.
    \item \textbf{Adversarial KD} adds an adversarial discriminator between teacher and student slots, aiming to further reduce differences between them.
\end{itemize}

Moreover, we report results with various ablated versions or alternatives of \textsc{SlotMatch}, namely:
\begin{itemize}
%\vspace{-0.1cm}
    \item \textbf{Slot KD ($L_2$)} aligns teacher and student slot representations using $L_2$ instead of cosine similarity.
    \item \textbf{\textsc{SlotMatch} + Reconstruction KD} combines slot-level distillation and reconstruction output distillation.
\end{itemize}

\noindent \textbf{Hyperparameters.} We discuss reproducibility details and hyperparameter settings in the supplementary material.

\vspace{-0.1cm}
\subsection{Results}
\vspace{-0.1cm}

\noindent \textbf{Quantitative results.}
In Table~\ref{table:sota_main}, we present comparative results on MOVi-E, YTVIS-2021 and DAVIS. Remarkably, \textsc{SlotMatch} consistently outperforms all prior methods across these datasets, including its teacher, \textsc{SlotContrast}~\cite{Manasyan-CVPR-2025}. Despite having $3.6\times$ fewer parameters and up to $2.7\times$ faster inference speed than \textsc{SlotContrast} (as per Table~\ref{table:ablation_timings}), our student achieves higher mask precision (mBO) and comparable or better object separation (FG-ARI). Notably, we improve video mBO on YTVIS from 32.4 (for \textsc{SlotContrast}) to 32.8 (for \textsc{SlotMatch}), while reducing latency by nearly half, showing that accurate object-centric segmentation does not require large models. We find that mBO improvements are very large on the real-world DAVIS dataset, where the teacher may overfit to slot assignment, but underperform on precise masks. Interestingly, our student avoids this overfitting and better balances separation and alignment. This aligns with our hypothesis that slot supervision can offer better generalization than end-to-end training with larger capacity alone.

\begin{figure}[!t]
  \centering
  \includegraphics[width=1.0\linewidth]{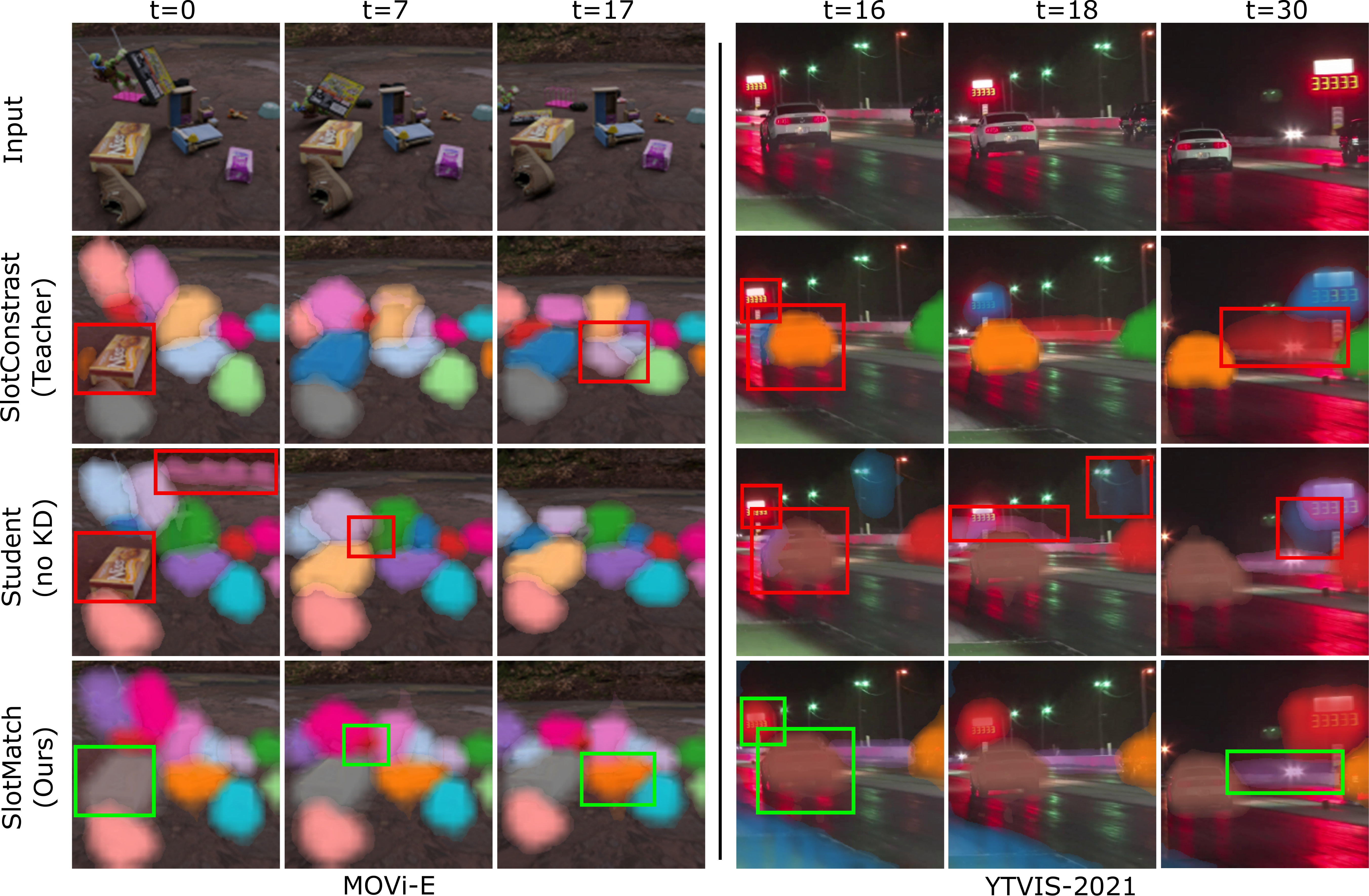}
      \vspace{-0.65cm}
    \caption{Qualitative segmentation results on MOVi-E (left) and YTVIS-2021 (right). The first row shows raw frames; the second and third rows show slots from the student and teacher models, respectively; the final row presents results from our distillation-based \textsc{SlotMatch}. \textsc{SlotMatch} recovers missed slots, refines object boundaries, and produces sharper, more consistent slots. Mistakes by the student and teacher models are annotated in {\color{red}red}, while corrections and additional detections introduced by \textsc{SlotMatch} are highlighted in {\color{Green}green}. Best viewed in color.} 
  \label{fig:visual_results}
      \vspace{-0.2cm}
\end{figure}

\noindent \textbf{Qualitative results.}
In Figure \ref{fig:visual_results}, we present a qualitative comparison of segmentation results produced by the teacher, the student model trained without distillation, and our \textsc{SlotMatch} model, across MOVi-E \cite{ghorbani2021movi} (left-hand side) and YTVIS-2021 \cite{vis2021} (right-hand side). 
On MOVi-E, which consists of synthetic scenes with numerous objects, both teacher and student models fail to detect certain objects in the initial video frames, \eg~the box on the left at $t=0$. 
The teacher model suffers from over-clustering (\eg~the object on the right at $t=17$), fragmenting single objects into multiple slots, while the student model (without distillation) fails to assign smaller or finer objects to any slot, \eg~the one in the center at $t=7$. In contrast, our student trained with \textsc{SlotMatch} consistently recovers missed instances, resolves the over-clustering issue, produces robust slot representations, and maintains temporally consistent slot assignments, \eg~consistent coloring of objects from $t=7$ to $t=17$. 
On YTVIS, a real-world and more challenging dataset with occlusions and poor lighting, we observe similar trends, \ie~both teacher and student models exhibit over-clustering in the initial frames, fragmenting single objects into multiple slots. In contrast, \textsc{SlotMatch} effectively refines the spatial extent of the masks and mitigates over-segmentation. Furthermore, our method outperforms in both slot assignment (\eg~correctly grouping the street light on the left) and boundary delineation (\eg~the guardrail on the right).
Green boxes indicate new or corrected slots introduced by our method, while red boxes highlight limitations of both the teacher and non-distilled student. These visual improvements confirm that distillation from the slots of a strong teacher not only preserves the original slot grouping, but also enhances the student's robustness to occlusion, lighting variation, and clutter, all in an unsupervised fashion.  

\begin{table}[t]
\centering
\setlength{\tabcolsep}{1.8pt}
\resizebox{0.9\linewidth}{!}{
\begin{tabular}{|l|c|c|c|c|c|}
\hline
\multirow{2}{*}{Method} & \#Params & \multirow{2}{*}{\#GFLOPs $\downarrow$} & \multicolumn{3}{c|}{Inference Time (ms) $\downarrow$} \\
\cline{4-6}
& (M) $\downarrow$ & & $\;\;$A100$\;\;$ &\;\;$3090$\;\;& \;\;$2080$\;\;\\
\hline\hline
Teacher & 91.0 & 8825.7 & 347.4 & 700.5 & 916.5 \\
\textsc{SlotMatch} & 25.2 & 3266.5 & 186.0 & 261.4 & 331.0 \\
\hline
Reduction & \textbf{3.61$\times$} & \textbf{2.70$\times$} & \textbf{1.87$\times$} & \textbf{2.68$\times$} & \textbf{2.77$\times$} \\
\hline
\end{tabular}
}
\vspace{-0.3cm}
\caption{Comparing inference speed, number of parameters, and GFLOPs between the \textsc{SlotContrast} teacher and our \textsc{SlotMatch} student. Measurements are performed on three GPUs (NVIDIA A100 with 40GB VRAM, NVIDIA 3090 with 24GB VRAM, NVIDIA 2080 with 12GB VRAM), using video samples from the YTVIS-2021 dataset. Inference times correspond to a \emph{whole video of average temporal length} (30 frames).}
\label{table:ablation_timings}
\vspace{-0.4cm}
\end{table}

\noindent \textbf{Efficiency comparison.}
In addition to improving segmentation quality, \textsc{SlotMatch} significantly reduces computational cost compared with the \textsc{SlotContrast} teacher. As shown in Table~\ref{table:ablation_timings}, our distilled student achieves a reduction of $3.6\times$ in parameters (91M$\rightarrow$25M) and $2.7\times$ in FLOPs, respectively. We run our experiments on three GPUs, to assess running time differences in distinct compute environments. We find that our student is between $1.9\times$ and $2.7\times$ faster than its teacher. The reported gains come with no additional supervision and without sacrificing slot quality or temporal consistency. This highlights the practical benefits of our distillation approach in resource-constrained settings. %, enabling object-centric video understanding on lightweight hardware.

\begin{table}[t]
\centering
\resizebox{0.86\linewidth}{!}{
\setlength{\tabcolsep}{2.9pt} 
\begin{tabular}{|l|cc|cc|}
\hline
\multirow{2}{*}{Method} & \multicolumn{2}{c|}{Image} & \multicolumn{2}{c|}{Video} \\
\cline{2-5}
 & FG-ARI $\uparrow$ & mBO $\uparrow$ & FG-ARI $\uparrow$  & mBO $\uparrow$ \\
\hline
\hline
%Teacher & 58.8 & 28.4 & 40.0 & 23.4  \\
Student (no KD) & 54.6 & 24.9 & 34.6 & 21.5     \\
\textsc{SlotMatch}  & 55.8 & 25.5 & 34.8 & 21.5  \\
\hline
\end{tabular}
}
\vspace{-0.3cm}
\caption{Zero-shot generalization results on the OVIS dataset, using models trained only on YTVIS (see performance on original dataset in Table \ref{table:sota_main}). We compare the student trained from scratch (no KD) versus our \textsc{SlotMatch} student.} 
\label{table:ytvis_to_ovis}
\vspace{-0.5cm}
\end{table}

\begin{table*}[t]
\centering
\setlength{\tabcolsep}{1.8pt}
\resizebox{0.98\linewidth}{!}{%
\begin{tabular}{|l|cc|cc|cc|cc|cc|cc|}
\hline
\multirow{3}{*}{Method}  & \multicolumn{4}{c|}{MOVi-E} & \multicolumn{4}{c|}{YTVIS-2021} & \multicolumn{4}{c|}{DAVIS 2017} \\ 
\cline{2-13}
& \multicolumn{2}{c|}{Image} & \multicolumn{2}{c|}{Video} & \multicolumn{2}{c|}{Image} & \multicolumn{2}{c|}{Video}  & \multicolumn{2}{c|}{Image} & \multicolumn{2}{c|}{Video} \\ 
\cline{2-13}
& \rotatebox{0}{FG-ARI $\uparrow\,$} & \rotatebox{0}{mBO $\uparrow$} & \rotatebox{0}{FG-ARI $\uparrow$} & \rotatebox{0}{mBO $\uparrow$} & \rotatebox{0}{FG-ARI $\uparrow$} & \rotatebox{0}{mBO $\uparrow$} & \rotatebox{0}{FG-ARI $\uparrow$} & \rotatebox{0}{mBO $\uparrow$} & \rotatebox{0}{FG-ARI $\uparrow$} & \rotatebox{0}{mBO $\uparrow$} & \rotatebox{0}{FG-ARI $\uparrow$} & \rotatebox{0}{mBO $\uparrow$}  \\ %
\hline
\hline
Student (no KD)
& 80.1 & 28.9 & 76.6 & \secondbest{\textbf{30.1}}
& 44.5 & 38.8 & \secondbest{\textbf{37.2}} & 32.1
& 93.7 & 11.9 & 76.6 & 11.6 \\ 
Feature KD
& 81.8 & \secondbest{\textbf{31.6}} & 80.1 & 28.1
& 45.2 & 38.2 & 35.9 & 31.4
& 94.6 & 12.8 & \secondbest{\textbf{91.6}} & 12.7 \\ %
Reconstruction KD
& 81.4 & 31.5 & \secondbest{\textbf{81.1}} & 29.1
& 44.9 & 38.8 & 35.6 & 31.8
& 92.0 & \secondbest{\textbf{20.5}} & 88.4 & \secondbest{\textbf{20.3}} \\ %
{Contrastive KD}
& 20.3 & 26.6 & 20.3 & 25.6
& 44.9 & 38.7 & 35.5 & 31.2
& 94.1 & 14.8 & 88.4 & 13.2 \\ %
{Adversarial KD}
& 11.8 & 22.5 & 20.4 & 14.9
& 13.3 & 16.8 & 15.0 & 15.4
&  93.1 & 11.2 & 88.0 & 11.1 \\ %

{Slot KD ($L_2$)}
& \secondbest{\textbf{82.1}} & 30.2 & 78.6 & 26.4
& 45.1 & \secondbest{\textbf{39.7}} & 36.9 & \secondbest{\textbf{32.6}}
& 94.4 & 20.0 & 89.6 & 20.0 \\ %
\hline

\textsc{SlotMatch} + Reconstruction KD
& 74.4 & 29.7 & 66.2 & 24.1
& \secondbest{\textbf{45.6}} & 38.3 & 36.1 & 31.1
& \best{\textbf{96.4}} & 12.7 & 91.2 & 12.4 \\ %
\hline
\rowcolor{TableColor} \textsc{SlotMatch} (ours)
& \best{\textbf{84.1}} & \best{\textbf{33.6}} & \best{\textbf{81.8}} & \best{\textbf{30.5}}
& \best{\textbf{45.8}} & \best{\textbf{39.8}} & \best{\textbf{37.3}} & \best{\textbf{32.8}}
& \secondbest{\textbf{95.6}} & \best{\textbf{30.6}} & \best{\textbf{92.8}} & \best{\textbf{31.6}} \\
\hline
\end{tabular}
}
\vspace{-0.3cm}
\caption{Distillation baselines and ablations on MOVi-E, YTVIS-2021, and DAVIS 2017. We report FG-ARI and mBO at both image and video levels. The top two scores for each metric are highlighted in \best{\textbf{blue bold}} (best) and \secondbest{\textbf{orange bold}} (second), respectively.}
\label{table:distill_ablations}
\vspace{-0.4cm}
\end{table*}

\noindent \textbf{Zero-shot generalization.}
In Table \ref{table:ytvis_to_ovis}, we compare \textsc{SlotMatch} against the student trained from scratch (without distillation), in the zero-shot setup on the challenging OVIS dataset. Notably, \textsc{SlotMatch} maintains competitive video-level performance (FG-ARI: 34.8 vs.~34.6) compared with the student trained from scratch, while preserving superior spatial mask quality (mBO: 25.5 vs.~24.9). This demonstrates that distilling structured slot representations enhances robustness to occlusion and domain shift. 

\noindent \textbf{Additional experiments.} In the supplementary material, we present experiments where \textsc{SlotMatch} successfully distills from other teachers. Moreover, we offer results regarding supervised DAVIS localization and counting tasks, a ViT-tiny student, a stronger-reconstruction-target control, and detailed slot-norm and matching analyses.

\vspace{-0.1cm}
\subsection{Ablation Studies}
\label{sec:ablations}
\vspace{-0.1cm}

We conduct ablations to isolate the contribution of each component in our framework. The ablation results are summarized in Tables~\ref{table:distill_ablations}, \ref{table:weight_distil}, and \ref{table:hungarian_matching}.

\noindent \textbf{\textsc{SlotMatch} vs.~alternative KD strategies.}
To test whether slot representations are the best target for distillation, we compare \textsc{SlotMatch} with variants that distill encoder features (feature KD) or decoder outputs (reconstruction KD), as shown in Table~\ref{table:distill_ablations}. While feature KD and reconstruction KD bring modest gains over the student (no KD) baseline, they remain below our slot-based formulation, by significant margins. We also evaluate \textsc{SlotMatch} against other KD strategies for slot distillation, including contrastive and adversarial KD. The results confirm that \textsc{SlotMatch} is more powerful than competing KD strategies. Overall, we find that distilling object-centric representations via \textsc{SlotMatch} leads to better temporal consistency and segmentation quality.

\noindent \textbf{Cosine vs.~$\mathbf{L_2}$ loss.}
We compare the proposed loss based on cosine similarity with an alternative loss based on $L_2$ loss over slot representations (see Table~\ref{table:distill_ablations}). Cosine-based distillation yields more stable improvements across datasets, especially on YTVIS, likely due to its scale invariance and better alignment with semantic structure in high-dimensional spaces.

\begin{table}[t]
\centering
\resizebox{0.8\linewidth}{!}{
\setlength{\tabcolsep}{2.8pt} 
\begin{tabular}{|l|c|cc|cc|}
\hline
\multirow{2}{*}{Method} & \multirow{2}{*}{$\beta$} & \multicolumn{2}{c|}{Image} & \multicolumn{2}{c|}{Video} \\
\cline{3-6}
 &  & FG-ARI $\uparrow$ & mBO $\uparrow$ & FG-ARI $\uparrow$  & mBO $\uparrow$ \\
\hline
\hline
Teacher & - & 45.5 & 39.7 & 36.8 & 32.4  \\
Student (no KD) & - & 44.5 & 38.8 & 37.2 & 32.1  \\
\hline
\multirow{5}{*}{\textsc{SlotMatch}} & 0.1 & 44.5 & 39.1 & 34.9 & 32.5 \\
 & 0.2 & \cellcolor{TableColor}  $\!\!$\textbf{45.8} & \cellcolor{TableColor} $\!\!$\textbf{39.8} & \cellcolor{TableColor}  $\!\!$\textbf{37.3} & \cellcolor{TableColor}  $\!\!$\textbf{32.8} \\
 & 0.3 & 44.4 & 39.1 & \textbf{36.3} & 32.5 \\
 & 0.5 & 43.8 & 38.9 & 35.5 & 32.1 \\
 & 0.8 & 43.9 & 38.3 & 35.4 & 32.0  \\
\hline
\end{tabular}
}
\vspace{-0.3cm}
\caption{Ablation study on the YTVIS dataset showing the effect of varying the weight $\beta$ of the distillation loss $\mathcal{L}_{\text{slot-KD}}$ in \textsc{SlotMatch}. The best performance across all metrics is achieved at $\beta = 0.2$.}
\label{table:weight_distil}
\vspace{-0.2cm}
\end{table}

\noindent \textbf{With or without reconstruction KD.} To empirically determine if the proposed $\mathcal{L}_{\text{slot-KD}}$ loss is sufficient or not, we carry out experiments with an enhanced version of \textsc{SlotMatch}, where the distillation is also applied over reconstructed features via the $\mathcal{L}_{\text{rec-KD}}$ loss (see Table~\ref{table:distill_ablations}). Perhaps surprisingly, this double distillation procedure degrades performance by considerable margins on MOVi-E. The results show that adding additional distillation losses can make the distillation more difficult to tune, eventually leading to inferior results. In summary, the empirical results concur with the implications of Theorem \ref{prop_slot_distill}, suggesting that $\mathcal{L}_{\text{slot-KD}}$ is sufficient for effective distillation.

\noindent \textbf{Distillation weight.}
We sweep the distillation loss weight $\beta$ and obtain the best performance at $\beta=0.2$ (see Table~\ref{table:weight_distil}). Too little weight underutilizes the teacher signal, while too much harms diversity by forcing alignment too strongly.

\begin{table}[t]
\centering
\resizebox{0.96\linewidth}{!}{
\setlength{\tabcolsep}{3pt} 
\begin{tabular}{|l|cc|cc|}
\hline
\multirow{2}{*}{Method} & \multicolumn{2}{c|}{Image} & \multicolumn{2}{c|}{Video} \\
\cline{2-5}
 & \rotatebox{0}{FG-ARI $\uparrow\,$} & \rotatebox{0}{mBO $\uparrow$} & \rotatebox{0}{FG-ARI $\uparrow$}  & \rotatebox{0}{mBO $\uparrow$} \\
\hline
\hline
\textsc{SlotContrast} (Teacher) & 83.9 & 32.4 & 81.7 & 28.6 \\
\hline
Student (no KD) & 80.1 & 28.9 & 76.6 & 30.1 \\
\textsc{SlotMatch} (w/ matching) & 83.5 & 31.8 & 81.5 & 28.6 \\
\textsc{SlotMatch} (w/o matching) & \textbf{84.1} & \textbf{33.6} & \textbf{81.8} & \textbf{30.5} \\
\hline
\end{tabular}
}
\vspace{-0.3cm}
\caption{Comparison of distillation performance with and without Hungarian matching on the MOVi-E dataset. Omitting Hungarian matching (as proposed) leads to better performance.}
\label{table:hungarian_matching}
\vspace{-0.5cm}
\end{table}

\noindent \textbf{Hungarian vs.~direct index matching.}
When transferring slots, there are two viable assumptions: (A) teacher-student slot correspondence can be learned during knowledge transfer, and (B) alignment between teacher and student slots needs to be performed prior to the learning stage. To this end, we explore two slot matching strategies: (i) aligning slots by index positions (based on hypothesis A), or (ii) using the Hungarian algorithm to find the most similar slot pairs based on their features, before applying the optimization (based on hypothesis B). As shown in Table~\ref{table:hungarian_matching}, distillation without Hungarian matching yields better performance, while also being more computationally efficient. A likely cause for this outcome is that strict Hungarian-based matching over-constrains the learning process, negatively impacting the convergence of the distillation model and leading to suboptimal results. Therefore, all results reported in the experiments omit the Hungarian matching step.

\noindent \textbf{Student vs.~teacher slot norms.} In the supplementary, we provide evidence indicating that student slots generally exhibit lower norms than teacher slots.

\vspace{-0.2cm}
\section{Conclusion}
\vspace{-0.1cm}

In this work, we introduced \textsc{SlotMatch}, a simple and effective framework for distilling slot-based object representations from a large teacher model into a lightweight student. By aligning slots directly via cosine similarity, our method avoids auxiliary objectives. We provided theoretical and empirical evidence suggesting that this objective is sufficient to transfer semantic structure, leading to a student that outperforms its teacher in segmentation quality on MOVi-E, YTVIS and DAVIS, while being significantly more efficient. In some cases, our results established new state-of-the-art performance levels among unsupervised slot-based video models. In future work, we aim to continue our research in scaling down object-centric models for real-world deployment.

\vspace{-0.2cm}
\section{Acknowledgments}
\vspace{-0.1cm}
This research is supported by the project ``Romanian Hub for Artificial Intelligence - HRIA'', Smart Growth, Digitization and Financial Instruments Program, 2021-2027, MySMIS no.~351416, and Innovation Fund Denmark, Grant No.~2081-00001B.

{
    \small
    \bibliographystyle{ieeenat_fullname}
    \bibliography{main}
}

\setcounter{section}{0}
\renewcommand{\thesection}{S\arabic{section}}

\setcounter{table}{0}
\renewcommand{\thetable}{S\arabic{table}}

\setcounter{figure}{0}
\renewcommand{\thefigure}{S\arabic{figure}}

% It is strongly recommended to use hyperref, especially for the review version.
% hyperref with option pagebackref eases the reviewers' job.
% Please disable hyperref *only* if you encounter grave issues, 
% e.g. with the file validation for the camera-ready version.
%
% If you comment hyperref and then uncomment it, you should delete *.aux before re-running LaTeX.
% (Or just hit 'q' on the first LaTeX run, let it finish, and you should be clear).

\section{Supplementary}

In the supplementary, we include the formal algorithm, the demonstration for Theorem 1, additional quantitative and qualitative results, as well as reproducibility details.

\subsection{Formal Algorithm}

Training proceeds as formally described in Algorithm~\ref{alg:slot_distillation}. For each mini-batch of $B$ videos with $T$ frames each, we extract per-frame slot representations $\boldsymbol{S}^{\mathbf{T}}_{b,t}$ and $\boldsymbol{S}^{\mathbf{S}}_{b,t}$ from the teacher and student models (steps 5 and 7), where $b$ is the video index and $t$ is the frame index. In step 8, a decoder reconstructs the input features from the student slots, producing $\hat{\boldsymbol{F}}{}^{\mathbf{S}}_{b,t}$. We next compute three loss terms: (i) a reconstruction loss $\mathcal{L}_{\text{rec}}$ to ensure that student slots retain sufficient scene-level information (step 9); (ii) a temporal contrastive loss $\mathcal{L}_{\text{slot-contrast}}$ to promote both consistency and diversity of slots across adjacent frames (step 10); and (iii) our core distillation objective $\mathcal{L}_{\text{slot-KD}}$ to align each student slot with its teacher counterpart via cosine similarity (step 11). These are combined into a weighted loss (step 12) used to update the student parameters via gradient descent (step 13), while the teacher remains frozen during the whole process.

\begin{algorithm*}[!t]
% \scriptsize
% \SetInd{0.2em}{0.4em}
% \setlength{\algomargin}{0.1em}
\linespread{0.9}\selectfont
\caption{\textsc{SlotMatch}: Slot-Level Knowledge Distillation}
\label{alg:slot_distillation}

\KwIn{
$\mathcal{D} = \{ \boldsymbol{X}_1, \dots, \boldsymbol{X}_{\!M} \}$ -- training set of $M$ video sequences, each with $T$ frames; \\
$\mathbf{T}=\mathbf{T}^{\text{enc}+\text{attn}+\text{dec}}$ -- pre-trained teacher with frozen weights $\theta_{\mathbf{T}}=\left\{\theta_{\mathbf{T}}^{\text{enc}}, \theta_{\mathbf{T}}^{\text{attn}}, \theta_{\mathbf{T}}^{\text{dec}} \right\}$; \\
$\mathbf{S}=\mathbf{S}^{\text{enc}+\text{attn}+\text{dec}}$ -- student model with trainable weights $\theta_{\mathbf{S}}=\left\{\theta_{\mathbf{S}}^{\text{enc}}, \theta_{\mathbf{S}}^{\text{attn}}, \theta_{\mathbf{S}}^{\text{dec}} \right\}$; \\
$\alpha$ -- contrastive loss weight, $\beta$ -- slot distillation weight, $\eta$ -- learning rate; $\tau$ -- temperature parameter;
$N$ -- number of slots per frame;
$B$ -- mini-batch size.
}

\KwOut{$\theta_{\textbf{S}}$ -- learned weights of the student model.}

\Repeat{convergence}{
    \ForEach{mini-batch $\mathcal{B} = \left\{ \boldsymbol{V}^{(b)} \right\}_{b=1}^B \subset \mathcal{D}$}{
        \ForEach{video index $b \in \{1,...,B\}$}{
            \ForEach{frame index $t \in \{1,...,T\}$}{

                $\boldsymbol{S}^{\mathbf{T}}_{b,t} \gets \mathbf{T}^{\text{enc}+\text{attn}}\left(\boldsymbol{x}^{(b)}_t; \theta_{\mathbf{T}}^{\text{enc}}, \theta_{\mathbf{T}}^{\text{attn}}\right)$ \textcolor{commentGreen}{$\lhd$ obtain teacher slots of shape $N \times d$}\\

                ${\boldsymbol{F}}{}^{\mathbf{S}}_{b,t} \gets \mathbf{S}^{\text{enc}}\left(\boldsymbol{x}^{(b)}_t; \theta_{\mathbf{S}}^{\text{enc}}\right)$ \textcolor{commentGreen}{$\lhd$ obtain student features}\\
                                
                $\boldsymbol{S}^{\mathbf{S}}_{b,t} \gets \mathbf{S}^{\text{attn}}\left({\boldsymbol{F}}{}^{\mathbf{S}}_{b,t}; \theta_{\mathbf{S}}^{\text{attn}}\right)$ \textcolor{commentGreen}{$\lhd$ obtain student slots of shape $N \times d$}\\
                                
                $\hat{\boldsymbol{F}}{}^{\mathbf{S}}_{b,t} \gets \mathbf{S}^{\text{dec}}\left(\boldsymbol{S}^{\mathbf{S}}_{b,t}; \theta_{\mathbf{S}}^{\text{dec}}\right)$ \textcolor{commentGreen}{$\lhd$ obtain decoded reconstruction}\\
            }
        }

        $\mathcal{L}_{\text{rec}} \gets \frac{1}{B\cdot T} \sum_{b=1}^{B} \sum_{t=1}^{T} \left\|\hat{\boldsymbol{F}}{}^{\mathbf{S}}_{b,t} - \boldsymbol{F}{}^{\mathbf{S}}_{b,t}\right\|^2$ \textcolor{commentGreen}{$\lhd$ feature reconstruction loss (based on MSE)}\\

        $\mathcal{L}_{\text{slot-contrast}} \gets \frac{1}{B\cdot T\cdot N} \sum_{b=1}^{B} \sum_{t=1}^{T-1} \sum_{n=1}^{N} 
        -\log \frac{
        \exp \left( \text{sim}(s^{\mathbf{S}}_{b,t,n},\, s^{\mathbf{S}}_{b,t+1,n})/\tau \right)}{
        \sum_{b',t',n'} \exp \left( \text{sim}(s^{\mathbf{S}}_{b,t,n},\, s^{\mathbf{S}}_{b',t',n'})/\tau \right)}$ \textcolor{commentGreen}{$\lhd$ slot-slot contrastive loss}\\

        $\mathcal{L}_{\text{slot-KD}} \gets \frac{1}{B\cdot T\cdot N} \sum_{b=1}^{B} \sum_{t=1}^{T} \sum_{n=1}^{N} \left(1 - 
        \frac{
        \langle s^{\mathbf{S}}_{b,t,n},\, s^{\mathbf{T}}_{b,t,n} \rangle}{
        \|s^{\mathbf{S}}_{b,t,n}\| \cdot \|s^{\mathbf{T}}_{b,t,n}\|} \right)$ \textcolor{commentGreen}{$\lhd$ distillation loss (based on cosine similarity)}\\

        $\mathcal{L}_{\text{total}} \gets \mathcal{L}_{\text{rec}} + \alpha \cdot \mathcal{L}_{\text{slot-contrast}} + \beta \cdot \mathcal{L}_{\text{slot-KD}}$ \textcolor{commentGreen}{$\lhd$ compute combined training objective}\\

        $\theta_{\textbf{S}} \gets \theta_{\textbf{S}} - \eta \cdot \nabla_{\!\theta_{\textbf{S}}}\, \mathcal{L}_{\text{total}}$ \textcolor{commentGreen}{$\lhd$ update the weights of the student model}\\
    }
}

\end{algorithm*}

\subsection{Theoretical Demonstration}

We next provide the proof of our theorem. For the sake of completion, we restate the theorem below.

\begin{theorem}\label{prop_slot_distill}
Let $s^\mathbf{T}\!\in \mathbb{R}^d$ be a teacher slot and $s^\mathbf{S}\!\in \mathbb{R}^d$ a student slot, with $\left\|s^\mathbf{T}\right\| > 0$ and $\left\|s^\mathbf{S}\right\| > 0$. Let $f\!: \mathbb{R}^d \to \mathbb{R}^m$ be a $K_f$-Lipschitz neural network that decodes the slots into features. If the slot distillation loss converges to a constant $c$, \ie:
\begin{equation}\label{eq_slot_distill}
\mathcal{L}_{\text{slot-KD}}\left(s^\mathbf{T},s^\mathbf{S}\right)=1 - \frac{\langle s^\mathbf{T}, s^\mathbf{S} \rangle}{\|s^\mathbf{T}\|\cdot\|s^\mathbf{S}\|} = c,
\end{equation}
then:
\begin{equation}\label{eq_feat_distill}
\begin{split}
\mathcal{L}_{\text{rec-KD}}&\!\left(f\!\left(s^\mathbf{T}\right)\!,f\!\left(s^\mathbf{S}\right)\!\right)\!=\! \left\|f\!\left(s^\mathbf{T}\right)\!-\!f\!\left(s^\mathbf{S}\right)\!\right\|^2\!\\
&\leq K_f^2\!\cdot\!\left[ \left(\|s^\mathbf{T}\|\!-\!\|s^\mathbf{S}\|\right)^2\!+\!2\!\cdot\! \|s^\mathbf{T}\|\!\cdot\! \|s^\mathbf{S}\|\! \cdot\! c \right].
\end{split}
\end{equation}
\end{theorem}

\begin{proof}
The cosine similarity between two (teacher and student) slots is defined as:
\begin{equation}\label{eq_cos_sim}
\cos \theta = \frac{\langle s^\mathbf{T}\!,s^\mathbf{S} \rangle}{\|s^\mathbf{T}\| \cdot \|s^\mathbf{S}\|}.
\end{equation}

By replacing the definition from Eq.~\eqref{eq_cos_sim} in Eq.~\eqref{eq_slot_distill}, we obtain the following:
\begin{equation}
\mathcal{L}_{\text{slot-KD}}\left(s^\mathbf{T},s^\mathbf{S}\right) = 1 - \cos \theta = c,
\end{equation}
which implies that:
\begin{equation}\label{eq_cos_rel}
\cos \theta = 1 - c. 
\end{equation}

We next express the sum of squared errors (squared Euclidean distance) between slots $\left\|s^\mathbf{T}\!-\!s^\mathbf{S}\right\|^2$ in terms of $\theta$. We start from the following definition:
\begin{equation}\label{eq_euclid_relation}
\left\|s^\mathbf{T}\!-\!s^\mathbf{S}\right\|^2 = \left\|s^\mathbf{T}\right\|^2 + \left\|s^\mathbf{S}\right\|^2 - 2 \cdot \langle s^\mathbf{T}\!,s^\mathbf{S} \rangle.
\end{equation}
From Eq.~\eqref{eq_cos_sim}, the scalar product between slots can be written as follows:
\begin{equation}
\langle s^\mathbf{T}\!,s^\mathbf{S} \rangle = \|s^\mathbf{T}\| \cdot \|s^\mathbf{S}\| \cdot \cos \theta.
\end{equation}
Hence, Eq.~\eqref{eq_euclid_relation} becomes: 
\begin{equation}\label{eq_euclid_relation_updated}
\left\|s^\mathbf{T}\!-\!s^\mathbf{S}\right\|^2 = \left\|s^\mathbf{T}\right\|^2 + \left\|s^\mathbf{S}\right\|^2 - 2 \cdot \|s^\mathbf{T}\| \cdot \|s^\mathbf{S}\| \cdot \cos \theta.
\end{equation}
By substituting Eq.~\eqref{eq_cos_rel} in Eq.~\eqref{eq_euclid_relation_updated}, we obtain the following:
\begin{equation}\label{eq_euclid_relation_updated2}
\begin{split}
\left\|s^\mathbf{T}\!-\!s^\mathbf{S}\right\|^2& \!=\! \left\|s^\mathbf{T}\right\|^2\!\!+\! \left\|s^\mathbf{S}\right\|^2\!\!-\!2\! \cdot\! \|s^\mathbf{T}\| \!\cdot\! \|s^\mathbf{S}\| \!\cdot\! (1\! -\!c)\\
&\!\!\!\!\!\!\!\!\!\!=\!\left\|s^\mathbf{T}\right\|^2\!\!+\! \left\|s^\mathbf{S}\right\|^2\!\!-\! 2\! \cdot\! \|s^\mathbf{T}\| \!\cdot\! \|s^\mathbf{S}\|\!+\! 2\!\cdot\! \|s^\mathbf{T}\|\!\cdot\! \|s^\mathbf{S}\| \!\cdot\! c\\
&\!\!\!\!\!\!\!\!\!\!=\! \left(\|s^\mathbf{T}\|\!-\!\|s^\mathbf{S}\|\right)^2\! +\! 2\! \cdot\! \|s^\mathbf{T}\| \!\cdot\! \|s^\mathbf{S}\| \!\cdot\! c.
\end{split}
\end{equation}

% \begin{equation}\label{eq_dist_theta_rel_2}
% \left\|s^\mathbf{T}\!-\!s^\mathbf{S}\right\|^2 \leq 2 \cdot r^2 \cdot c.
% \end{equation}

The Lipschitz property of $f$ gives us the following inequality:
\begin{equation}\label{eq_lip_ineq}
\left\|f\!\left(s^\mathbf{T}\right)\!-\!f\!\left(s^\mathbf{S}\right)\!\right\|\!\leq K_f \cdot \left\|s^\mathbf{T}\!-\!s^\mathbf{S}\right\|.
\end{equation}
Squaring both sides leads to:
\begin{equation}\label{eq_lip_ineq_square}
\left\|f\!\left(s^\mathbf{T}\right)\!-\!f\!\left(s^\mathbf{S}\right)\!\right\|^2\!\leq K_f^2 \cdot \left\|s^\mathbf{T}\!-\!s^\mathbf{S}\right\|^2.
\end{equation}
Note that the left term in Eq.~\eqref{eq_lip_ineq_square} is equal to $\mathcal{L}_{\text{rec-KD}}$, \ie:
\begin{equation}\label{eq_bound}
\mathcal{L}_{\text{rec-KD}}\!\left(f\!\left(s^\mathbf{T}\right)\!,f\!\left(s^\mathbf{S}\right)\!\right)\leq K_f^2 \cdot \left\|s^\mathbf{T}\!-\!s^\mathbf{S}\right\|^2
\end{equation}
We next substitute Eq.~\eqref{eq_euclid_relation_updated2} inside Eq.~\eqref{eq_bound} and obtain:
\begin{equation}\label{eq_bound_2}
\mathcal{L}_{\text{rec-KD}}\!\leq\! K_f^2 \cdot \left[\! \left(\|s^\mathbf{T}\|\!-\!\|s^\mathbf{S}\|\right)^2\!\! +\! 2\! \cdot \!\|s^\mathbf{T}\| \!\cdot\! \|s^\mathbf{S}\| \!\cdot\! c \right]\!,
\end{equation}
% Let $K=2 \cdot K_f^2 \cdot r^2$. Finally, we obtain:
% \begin{equation}\label{eq_bound_3}
% \mathcal{L}_{\text{rec-KD}}\!\left(f\!\left(s^\mathbf{T}\right)\!,f\!\left(s^\mathbf{S}\right)\!\right)\leq K \cdot c,
% \end{equation}
which concludes the proof of Theorem \ref{prop_slot_distill}. %\ref{prop_slot_distill}.
\end{proof}

In neural networks, the constant $K_f$ depends on the weights and activations of the model, and it can be estimated via spectral norms. In our reconstruction KD formulation considered in Theorem \ref{prop_slot_distill}, the same comparison decoder, namely the student decoder, is applied to both teacher and student slots. Thus, $f(s^\mathbf{T})$ and $f(s^\mathbf{S})$ denote reconstructions obtained through a shared map $f$, and the Lipschitz argument applies. The teacher's original decoder is not used during distillation.

\begin{figure*}[!t]
  \centering
  \includegraphics[width=0.8\linewidth]{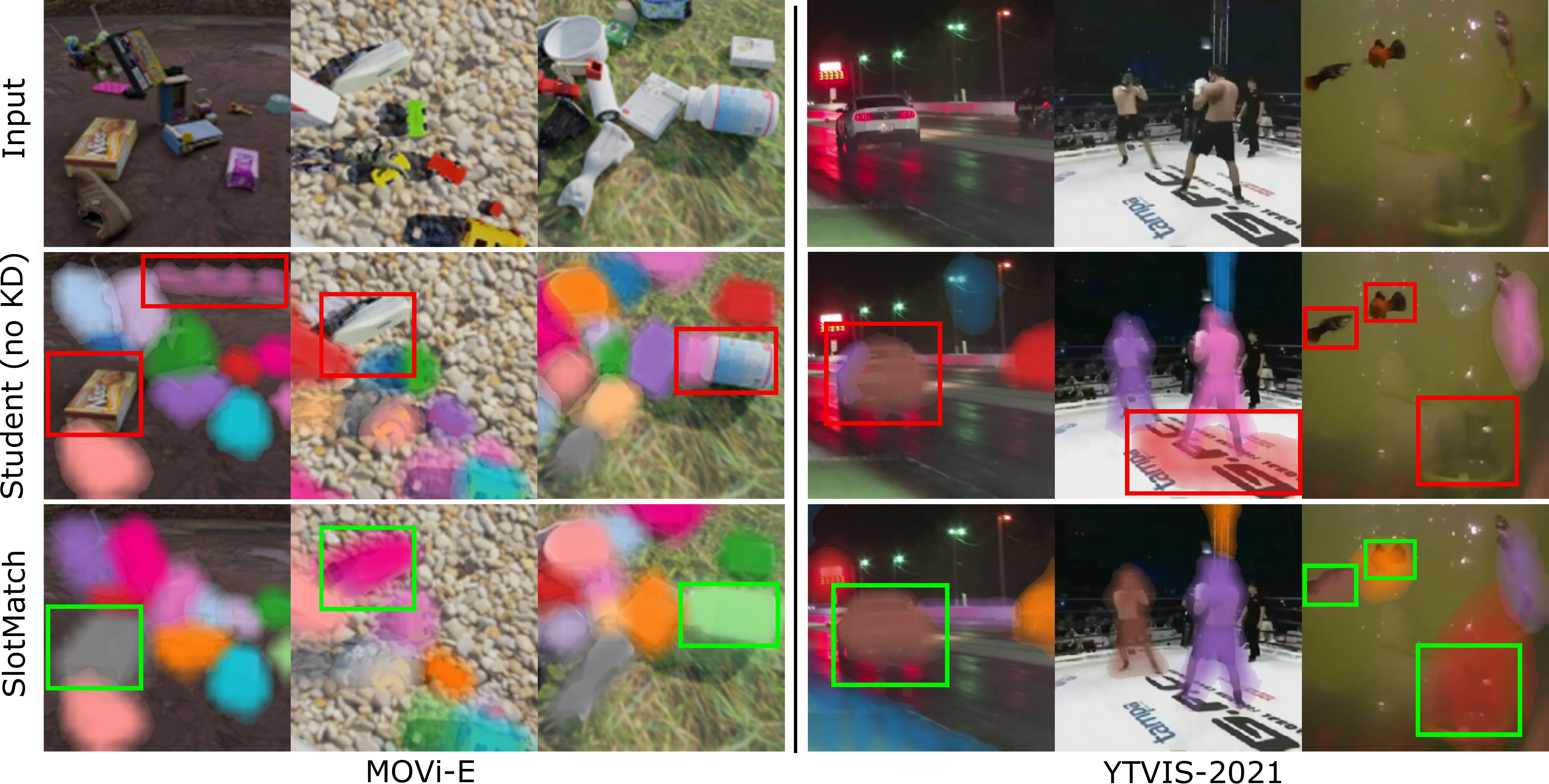}
      \vspace{-0.1cm}
  \caption{Qualitative comparison on MOVi-E (left) and YTVIS-2021 (right). The second row shows outputs from the student model (no KD), while the third row presents results from our distillation-based \textsc{SlotMatch}. Student (no KD) errors, including missed slots, are marked in {\color{red}red}. Corrections and additional slots introduced by \textsc{SlotMatch} are highlighted in {\color{Green}green}. Best viewed in color.} %, therefore, $M=L/2$.} %
  \label{fig:visual_results_supp}
      %\vspace{-0.2cm}
\end{figure*}

\begin{table}[t]
\centering
\small{
\setlength{\tabcolsep}{3pt} 
\begin{tabular}{|l|c|c|}
\hline
Method & {Image mBO $\uparrow$} & {Video mBO $\uparrow$} \\
\hline
\hline
VideoSAUR \cite{Zadaianchuk-NeurIPS-2023} (teacher) & \secondbest{\textbf{23.9}} & \secondbest{\textbf{18.4}} \\
Student (no KD) & 17.4 & 14.4 \\
\textsc{SlotMatch} (ours) & \best{\textbf{30.3}} & \best{\textbf{23.2}} \\
\hline
VideoSAURv2 \cite{Manasyan-CVPR-2025} (teacher) & \secondbest{\textbf{23.0}} & \secondbest{\textbf{19.2}} \\
 Student (no KD) & 22.9  & 18.2 \\
\textsc{SlotMatch} (ours) & \best{\textbf{31.2}} & \best{\textbf{28.9}} \\
\hline
\end{tabular}
}
\vspace{-0.2cm}
\caption{Video segmentation results on DAVIS \cite{Perazzi-CVPR-2016} with two different teacher models, VideoSAUR \cite{Zadaianchuk-NeurIPS-2023} and VideoSAURv2 \cite{Manasyan-CVPR-2025}. The top two scores for each metric are highlighted in \best{\textbf{blue bold}} (top method) and \secondbest{\textbf{orange bold}} (second best).}
\label{table:additional_results_teachers}
\end{table}

% Remarks:
% \begin{itemize}
%     \item If the slots are normalized (\ie~$r = 1$), the bound simplifies to:
% \begin{equation}\label{eq_bound_4}
% \mathcal{L}_{\text{rec-KD}}\!\left(f\!\left(s^\mathbf{T}\right)\!,f\!\left(s^\mathbf{S}\right)\!\right)\leq 2 \cdot K_f^2 \cdot c.
% \end{equation}
%     \item In neural networks, the constant $K_f$ depends on the weights and activations of the model, and it can be estimated via spectral norms.
% \end{itemize}

\subsection{Additional Results}

\noindent
\textbf{Qualitative results.} 
In Figure~\ref{fig:visual_results_supp}, we showcase qualitative comparisons of segmentation masks on challenging examples from MOVi-E and YTVIS. Our method produces sharper and more temporally-consistent masks than the student model without distillation. On MOVi-E, \textsc{SlotMatch} segments overlapping or partially occluded objects more cleanly, while on real-world YTVIS data, it shows improved boundary alignment and fewer slot collisions. These results highlight the benefit of slot-level supervision in guiding the student towards focusing on meaningful object structure.

\begin{table}[t]
\centering
\small{
\setlength{\tabcolsep}{3pt} 
\begin{tabular}{|l|c|c|}
\hline
Method & {Image FG-ARI $\uparrow$}  & {Image mBO $\uparrow$}\\
\hline
\hline
DINOSAURv2 \cite{Seitzer-arxiv-2022} (teacher) & \best{\textbf{44.7}} & \best{\textbf{29.3}} \\
Student (no KD) & 41.5 & 28.6 \\
\textsc{SlotMatch} (ours) & \secondbest{\textbf{42.3}} & \secondbest{\textbf{28.8}} \\
\hline
\end{tabular}
}
\vspace{-0.2cm}
\caption{Image segmentation results on MS COCO \cite{Lin-ECCV-2014} with DINOSAURv2 \cite{Seitzer-arxiv-2022} as teacher. The top two scores for each metric are highlighted in \best{\textbf{blue bold}} (top method) and \secondbest{\textbf{orange bold}} (second best).}
\label{table:additional_results_image}
%\vspace{-0.1cm}
\end{table}

\noindent
\textbf{Results with different teachers.}
In Table \ref{table:additional_results_teachers}, we report results with additional teacher models, namely VideoSAUR \cite{Zadaianchuk-NeurIPS-2023} and VideoSAURv2 \cite{Manasyan-CVPR-2025}, on the more challenging real-world DAVIS \cite{Perazzi-CVPR-2016} dataset. The \textsc{SlotMatch} student based on DINOv2 is compared with an equivalent architecture without distillation (no KD), as well as its corresponding teacher model. The results follow the same trend as in the main experiments, \ie~\textsc{SlotMatch} consistently outperforms both the student (no KD) and teacher models. These results confirm that \textsc{SlotMatch} generalizes across multiple teacher models.

\begin{figure*}[t]
    \centering
    \includegraphics[width=0.98\linewidth]{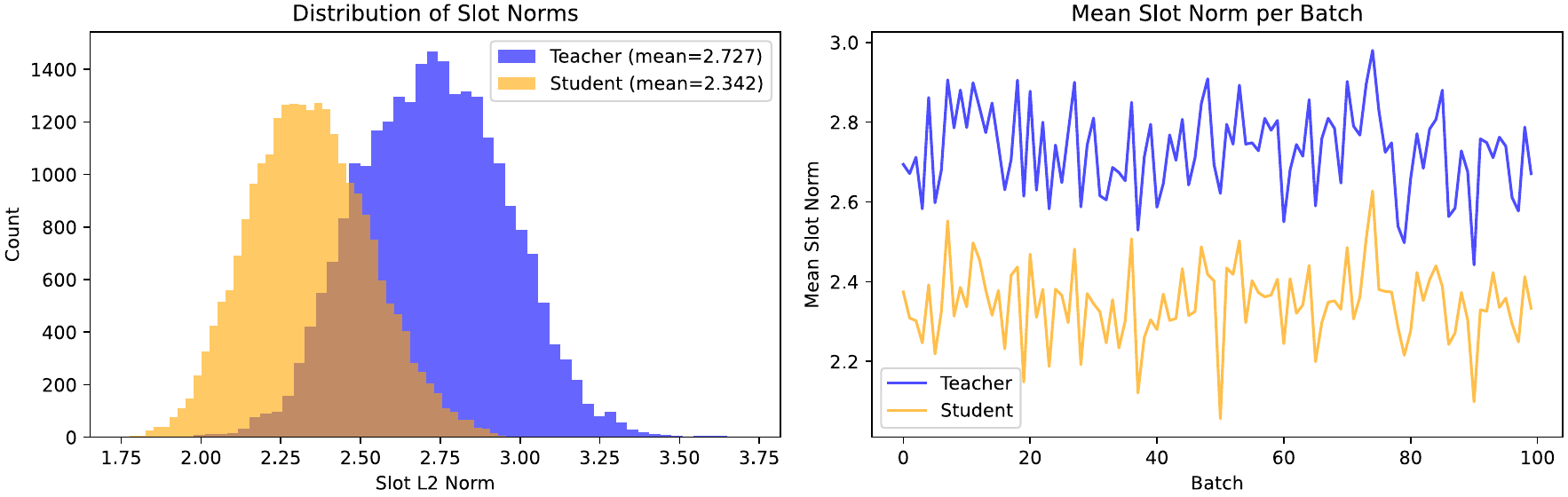}
    \vspace{-0.3cm}
    \caption{Distribution of slot $L_2$ slot norms for teacher and student models measured on YTVIS-2021 validation set. \textbf{Left:} Histogram of all slot norms showing that the student slot norm distribution is visibly shifted towards lower values, while preserving the same Gaussian shape. \textbf{Right:} Mean slot norm per batch, confirming the gap is stable across the entire dataset and not a local artifact. Best viewed in color.}
    \label{fig:slot_norms}
    \vspace{-0.1cm}
\end{figure*}

\begin{table}[t]
\centering
\small{
\setlength{\tabcolsep}{6pt}
\begin{tabular}{|l|cc|c|}
\hline
Metric & Teacher & Student & Ratio (S/T) \\
\hline\hline
Median & 2.7260 & 2.3350 & 0.8566 \\
Mean   & 2.7274 & 2.3416 & 0.8586 \\
Std.~dev.   & 0.2288 & 0.1887 & 0.8246 \\
% Min    & 1.9772 & 1.7506 & 0.8854 \\
% Max    & 3.7185 & 3.0077 & 0.8089 \\
\hline
\end{tabular}
}
\vspace{-0.2cm}
\caption{$L_2$ slot norm statistics for teacher and student models measured on the YTVIS-2021 validation set. The student consistently operates at a lower slot magnitude ($\approx 14\%$ lower), confirming that cosine-based distillation decouples directional alignment from magnitude.}
\label{tab:slot_norms}
%\vspace{-0.1cm}
\end{table}

\noindent
\textbf{Image segmentation results.} In Table \ref{table:additional_results_image}, we present image segmentation results on the MS COCO \cite{Lin-ECCV-2014} dataset,  employing yet another teacher, namely DINOSAURv2 \cite{Seitzer-arxiv-2022}. \textsc{SlotMatch} surpasses the analogous model (without distillation), while remaining slightly below the teacher. We conclude that \textsc{SlotMatch} can obtain competitive results in image segmentation, with a reduced computational cost. The true benefit of \textsc{SlotMatch} is more obvious in video segmentation, where its ability to maintain temporal consistency of slots plays a key role.

\noindent
\textbf{Hypothesis justification.} \textsc{SlotMatch} is based on the hypothesis that distilling knowledge via cosine similarity does not force the norms of the student slots to match the norms of the teacher slots. To confirm our hypothesis, we measure the $L_2$ norms of teacher and student slot representations on the YTVIS-2021 validation set. As shown in Figure \ref{fig:slot_norms}, the norms of the student slots are consistently lower than the norms of the teacher slots. The statistics reported in Table \ref{tab:slot_norms} further confirm that the ratio between student and teacher slots is consistently below 0.86. Hence, the slot norm measurements confirm that the cosine similarity loss deliberately decouples directional alignment from magnitude. Due to this systematic norm difference, \textsc{SlotMatch} outperforms the teacher on all datasets, suggesting that slot directions, not magnitudes, carry the semantically meaningful object-centric information. The generally larger norms of the teacher slots encode confidence overestimation artifacts tied to its higher-capacity encoder, which the student avoids by operating in a lower-magnitude regime. This supports our claim that cosine-based distillation acts as a bias filter, transferring good knowledge (slot directions), while discarding potentially harmful knowledge (slot magnitudes).

\begin{table}[t]
\centering
\setlength{\tabcolsep}{3pt}
\resizebox{1.0\columnwidth}{!}{%
\begin{tabular}{|l|ccc|}
\hline
\multirow{1}{*}{Method} 
& \rotatebox{0}{FG BBox IoU $\uparrow$}
& \rotatebox{0}{Count acc. $\uparrow$}
& \rotatebox{0}{GFLOPs $\downarrow$} \\
\hline
\hline
Fine-tune then distill & 0.424 & 0.799 & 17,415 \\
Distill then fine-tune (ours) & \best{\textbf{0.437}} & \best{\textbf{0.802}} & 3,015 \\
\hline
\end{tabular}
}
\vspace{-0.2cm}
\caption{Comparison of ``fine-tune then distill'' vs.~ ``distill then fine-tune'' on DAVIS 2017 localization and counting downstream tasks. FG BBox IoU evaluates overlap with the foreground bounding box enclosing all annotated objects. Count accuracy evaluates the predicted number of annotated foreground object IDs. The top score for each metric is highlighted in \best{\textbf{blue bold}}.}
\label{tab:distill_order_loc_count}
\end{table}

\begin{table}[t]
\centering
\setlength{\tabcolsep}{2.5pt}
\resizebox{1.0\columnwidth}{!}{%
\begin{tabular}{|l|cc|cc|}
\hline
\multirow{2}{*}{Method} & \multicolumn{2}{c|}{Image} & \multicolumn{2}{c|}{Video} \\
\cline{2-5}
 & FG-ARI $\uparrow$ & mBO $\uparrow$ & FG-ARI $\uparrow$  & mBO $\uparrow$ \\
\hline
\hline
Better target (DINOv2-B recon.) & 41.8 & 39.4 & 33.3 & 31.5 \\
\textsc{SlotMatch} (ours) & \best{\textbf{45.8}} & \best{\textbf{39.8}} & \best{\textbf{37.3}} & \best{\textbf{32.8}} \\
\hline
\end{tabular}
}
\vspace{-0.2cm}
\caption{Ablation of the reconstruction target on YTVIS-2021. The top score for each metric is highlighted in \best{\textbf{blue bold}}.}
\label{tab:ablation_target}
\end{table}

\begin{table}[!t]
  \centering
  \setlength\tabcolsep{0.38em}
  \resizebox{1.0\columnwidth}{!}{%
  \begin{tabular}{|l|cc|cc|}
  \hline
\multirow{2}{*}{Method} & \multicolumn{2}{c|}{Image} & \multicolumn{2}{c|}{Video} \\
\cline{2-5}
 & FG-ARI $\uparrow$ & mBO $\uparrow$ & FG-ARI $\uparrow$  & mBO $\uparrow$ \\
\hline
\hline
      \textsc{SlotContrast} (ViT-B)   & 89.1 & 12.1 & 73.3 & 11.8  \\
       % \textsc{SlotMatch} (ViT-tiny)   & 95.3 & 12.7 & 88.0 & 12.0 \\
       \textsc{SlotMatch} (ViT-tiny)   & 94.6 & 13.9 & 88.0 & 13.0 \\
       \textsc{SlotMatch} (ViT-S)   & 95.6 & 30.6 & 92.8 & 31.6 \\
  \hline
  \end{tabular}
  }
\vspace{-0.2cm}
  \caption{Results with ViT-tiny as student on DAVIS 2017.}
  \label{tab_davis_extra}
\end{table} 

\noindent \textbf{Application to downstream tasks.} While our unsupervised slot distillation leads to an equally effective but more efficient model, we are interested in showing its practical applicability on supervised downstream tasks. As such, we conduct downstream localization and counting experiments on DAVIS 2017, with supervision derived from the dataset masks. % We aim to compare  case,  % We add a supervised DAVIS localization and counting probe, with supervision derived from the dataset masks
The results reported in Table~\ref{tab:distill_order_loc_count} indicate that our ``distill then fine-tune'' pipeline leads to better performance levels than the conventional ``fine-tune then distill'' pipeline, on the supervised object localization and counting downstream tasks, while requiring $5.8\times$ less compute. We thus consider that it is inefficient to ``fine-tune then distill'' a large model, especially if training resources are limited. 

\begin{table}[t]
\centering
\resizebox{\columnwidth}{!}{
\begin{tabular}{|l|cc|cc|}
\hline
\multirow{2}{*}{Method} & \multicolumn{2}{c|}{Image} & \multicolumn{2}{c|}{Video} \\
\cline{2-5}
 & FG-ARI $\uparrow$ & mBO $\uparrow$
       & FG-ARI $\uparrow$ & mBO $\uparrow$ \\
\hline
\hline
\textsc{SlotContrast} (Teacher)       & 89.1 & 12.1 & 73.3 & 11.8 \\
Student (no KD)              & 93.7 & 11.9 & 76.6 & 11.6 \\
Slot KD (L2)                 & 94.4 & 20.0 & 89.6 & 20.0 \\
\textsc{SlotMatch} + Norm KD
                              & 91.8 & 20.3 & 88.0 & 17.9 \\
\textsc{SlotMatch} (ours)             & \textbf{95.6} & \textbf{30.6}
                              & \textbf{92.8} & \textbf{31.6} \\
\hline
\end{tabular}}
\vspace{-0.2cm}
\caption{Effect of explicitly constraining student slot magnitudes to match teacher slot magnitudes on DAVIS 2017. The constrained version (\textsc{SlotMatch} + Norm KD) integrates the loss defined in Eq.~\eqref{eq_norm_KD} in the training objective.}
\label{tab:norm_ablation}
\vspace{-0.2cm}
\end{table}

\noindent \textbf{Slot space vs.~better target.} The distillation plus the original losses conflates two hypotheses: (i) gains come from distilling in slot space, and (ii) gains come from aiming at a stronger feature target than what DINOv2-S can produce. To decouple these hypotheses, we conduct an experiment on YTVIS-2021 using target features from a stronger model (DINOv2-B), without distillation. As shown in Table~\ref{tab:ablation_target}, simply using DINOv2-B as a stronger reconstruction target without slot-level distillation is consistently below \textsc{SlotMatch}, confirming that the gains stem from slot-space supervision rather than the quality of the target alone.

\noindent \textbf{Distillation to a tiny student.} 
To test how far the distillation approach scales down, we conduct experiments with a very small student encoder, namely ViT-tiny, on DAVIS 2017. The results reported in Table~\ref{tab_davis_extra} show that the smaller student (ViT-tiny) still outperforms the teacher, which is remarkable.

\begin{table*}[t]
\centering
\small
\begin{tabular}{|l|c|c|c|c|c|c|}
\hline
\multirow{2}{*}{{Hyperparameter}} & \multicolumn{2}{c|}{{MOVi-E}} & \multicolumn{2}{c|}{{YTVIS-2021}} & \multicolumn{2}{c|}{{DAVIS 2017}} \\
\cline{2-7}
& {Teacher} & {Student} & {Teacher} & {Student} & {Teacher} & {Student} \\
\hline
\hline
Backbone & ViT-B/14 & ViT-S/14 & ViT-B/14 & ViT-S/14 & ViT-B/14 & ViT-S/14  \\
Feature size ($m$) & 768 & 384 & 768 & 384 & 768 & 384 \\
Slot dim ($S$) & 128 & 128 & 64 & 64 & 128 & 128 \\
\#Slots ($N$) & 15 & 15 & 7 & 7 & 4 & 4 \\
Input size & 336$\times$336 & 336$\times$336 & 518$\times$518 & 518$\times$518 & 336$\times$336 & 336$\times$336 \\
\#Patches & 576 & 576 & 1369 & 1369 & 576 & 576 \\
Batch size ($B$) & 16 & 16 & 64 & 64 & 16 & 16 \\
Learning rate ($\eta$) & 0.0004 & 0.0004 & 0.0008 & 0.0008 & 0.0008 & 0.0008 \\
Total steps & 300K & 300K & 100K & 100K & 100K & 100K \\
\cline{2-7}
Loss weights ($\alpha, \beta$) & \multicolumn{6}{c|}{(1, 0.2)} \\
Slot attention iterations & \multicolumn{6}{c|}{2} \\
Contrastive temperature ($\tau$) & \multicolumn{6}{c|}{0.1} \\
Gradient clip & \multicolumn{6}{c|}{0.05} \\
Predictor & \multicolumn{6}{c|}{Transformer (1$\times$4)} \\
\hline
\end{tabular}
\vspace{-0.2cm}
\caption{Summary of key hyperparameters for all teacher and student models used in our experiments. All models use the same decoder, predictor, and loss weights.}
\label{tab:all_hparams}
\vspace{-0.1cm}
\end{table*}
% \noindent \textbf{Statistical Significance Analysis on MOVi-E}
% To assess the reliability of performance differences between models on the MOVi-E dataset, we conduct both paired t-tests and Wilcoxon signed-rank tests over per-video metrics. Specifically, for each validation video, we compute the four metrics using the segmentation masks produced by the teacher, student trained from scratch, and distilled student models. We then perform statistical tests on the per-video scores to evaluate whether observed improvements are statistically significant.

% \begin{table}[t]
% \centering
% \small
% \setlength{\tabcolsep}{3pt}
% \begin{tabular}{|l|cccc|}
% \hline
% Comparison & ARI ($p$) & I-ARI ($p$) & mBO ($p$) & I-mBO ($p$) \\
% \hline
% \hline
% \textbf{t-test} \\
% Teacher vs. Student       & 0.4962 & 0.2786 & \textbf{0.0002} & \textbf{8.0e-05} \\
% Teacher vs. Distilled     & 0.4603 & 0.9901 & \textbf{1.1e-13} & \textbf{3.4e-15} \\
% Student vs. Distilled     & 0.8524 & 0.3395 & \textbf{3.7e-07} & \textbf{2.4e-07} \\
% \hline
% \textbf{Wilcoxon} \\
% Teacher vs. Student       & 0.1977 & 0.7677 & \textbf{0.0001} & \textbf{2.4e-05} \\
% Teacher vs. Distilled     & 0.4024 & 0.8067 & \textbf{3.0e-14} & \textbf{4.9e-16} \\
% Student vs. Distilled     & 0.3535 & 0.6862 & \textbf{1.2e-06} & \textbf{1.9e-06} \\
% \hline
% \end{tabular}
% \vspace{-0.2cm}
% \caption{Paired statistical significance tests ($p$-values) across model pairs using ARI and mBO metrics. Bolded values indicate statistically significant differences ($p < 0.05$).}
% \label{tab:statistical_tests_compact}
% \end{table}

\noindent \textbf{Adding explicit norm matching objective.} To directly test whether teacher slot magnitude contains useful information that should also be distilled, we add an explicit norm-matching objective to \textsc{SlotMatch}:
\begin{equation}\label{eq_norm_KD}
\mathcal{L}_{\text{norm-KD}}
=\!\frac{1}{B\!\cdot\!T\!\cdot\!N} \sum_{b=1}^{B}\!\sum_{t=1}^{T}\!\sum_{n=1}^{N}\!
\left(
\left\|s^{\mathbf{S}}_{b,t,n}\right\|
-
\left\|s^{\mathbf{T}}_{b,t,n}\right\|
\right)^2\!.
\end{equation}
We keep all other settings fixed and perform the experiments on DAVIS. The results reported in Table \ref{tab:norm_ablation} confirm that adding norm KD decreases all metrics, suggesting that explicitly forcing the student to reproduce the teacher's slot magnitudes degrades segmentation performance. This empirical observation directly supports transferring slot direction, while leaving student magnitude unconstrained. This complements the measured teacher vs.~student norm statistics in Table \ref{tab:slot_norms} (and Figure \ref{fig:slot_norms}), which show a systematic magnitude difference (lower student slot norms) via \textsc{SlotMatch} distillation.

\begin{figure}[t]
    \centering
    \includegraphics[width=1.0\linewidth]{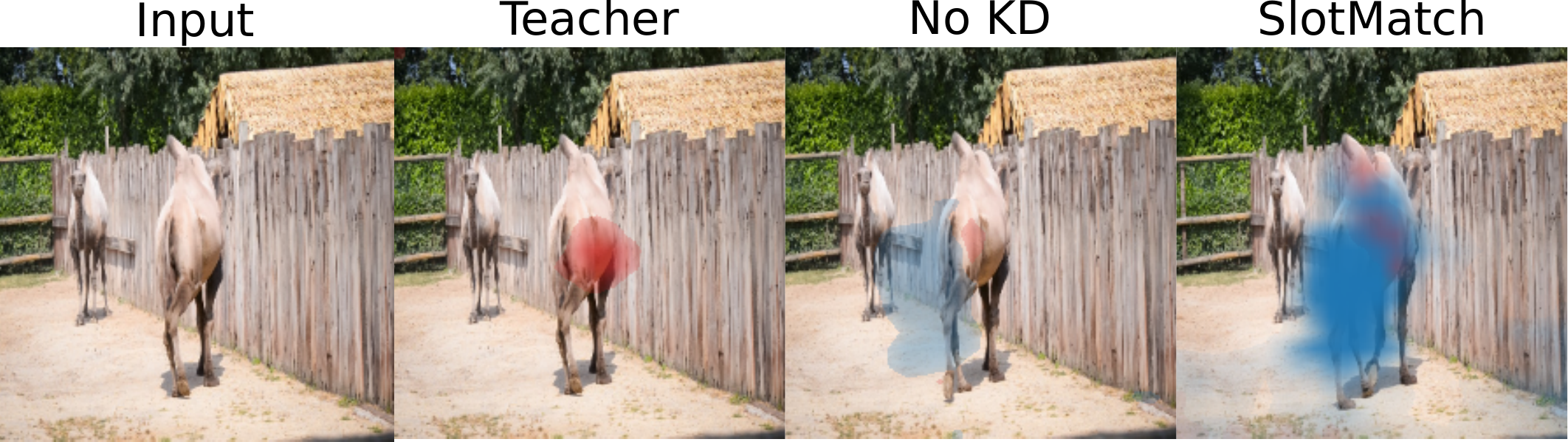}
    \vspace{-0.6cm}
    \caption{Slot visualization on DAVIS 2017. Best viewed in color.}
    \label{fig:davis_slots}
    \vspace{-0.2cm}
\end{figure}

\noindent \textbf{Explaining the large mBO on DAVIS.} In Figure~\ref{fig:davis_slots}, we show a representative case where a teacher slot captures a consistent, but partial region of the target object (the horse body, excluding the legs and the head). This partial capture maintains high FG-ARI due to consistent pixel grouping, while mBO remains low, as the predicted mask poorly overlaps with the horse. \textsc{SlotMatch} allocates a larger slot that better covers the object, thereby obtaining a better mBO.

\subsection{Limitations}

Despite its simplicity and effectiveness, \textsc{SlotMatch} has several limitations. First, its direct slot correspondence assumption may become unreliable when scenes contain severe occlusions, abrupt object appearances, or a varying number of objects. Second, although Theorem \ref{prop_slot_distill} shows that directional alignment can bound the discrepancy between decoded representations, it does not guarantee that ignoring slot norms is always beneficial. Moreover, the observed degradation caused by reconstruction distillation is empirical and may depend on the decoder architecture, loss weighting, and teacher-student capacity gap. Finally, our evaluation is limited to unsupervised video/image segmentation and a fixed family of slot-based architectures. Future work could address these limitations by selectively distilling norm information through confidence-aware regularization, dynamically balancing potentially conflicting objectives based on their gradients, and evaluating the method across additional architectures and downstream tasks (besides localization and counting).

\subsection{Reproducibility Details}

\noindent \textbf{Model configurations.}
In Table~\ref{tab:all_hparams}, we provide a complete overview of model and training hyperparameters used across all experiments. The teacher and student models share a common slot attention architecture, decoder, and loss structure. For both datasets, the student differs by using a ViT-S/14 encoder with lower feature dimensionality. Other parameters such as slot count, learning rate, and contrastive loss temperature are held constant. These settings enable consistent and reproducible evaluation of our \textsc{SlotMatch} distillation approach.

\noindent \textbf{Training time and compute resources.}
All experiments were conducted using two NVIDIA A100 GPUs (each with 40GB of VRAM). Training the student model on MOVi-E for 300K steps took approximately 138 hours. For YTVIS-2021, training required about 39 hours for 100K steps. %The teacher model (SlotContrast) was pre-trained separately and frozen during distillation. 
All models were implemented in PyTorch Lightning. Total GPU hours for each experiment are summarized in Table~\ref{tab:compute}.

\begin{table}[t]
\centering
\small
\setlength{\tabcolsep}{4pt}
\begin{tabular}{|l|l|c|c|}
\hline
{Method} & {Dataset} & {Steps} & {GPU Hours} \\
\hline
\hline
% \multirow{2}{*}{\textsc{SlotContrast}} & MOVi-E Teacher & 300k & ~143 \\
% & YTVIS Teacher & 100k & ~46 \\
% \hline
\multirow{3}{*}{\textsc{SlotMatch}} & MOVi-E & 300K & $\approx$ 138 \\
& YTVIS & 100K & $\approx$ 39 \\
& DAVIS & 100K & $\approx$ 25 \\
\hline
\end{tabular}
\vspace{-0.2cm}
\caption{Approximate training time per experiment on two A100 GPUs.}
\label{tab:compute}
\end{table}

\begin{table}[t]
\centering
\small{
\begin{tabular}{|c|c|cc|}
\hline
{Run} & {Seed} & {Video FG-ARI $\uparrow$}  & {Video mBO $\uparrow$} \\
\hline
\hline
Run 1 & 42 & 37.2 & 33.1 \\
Run 2 & 101  & 37.4 & 32.4 \\
Run 3 & 2048 & 37.3 & 32.9 \\
\hline
\textbf{Mean} & - & 37.3 & 32.8 \\
\textbf{Std.} & - & $\pm$0.06 & $\pm$0.29 \\
\hline
\end{tabular}
}
\vspace{-0.2cm}
\caption{Robustness evaluation of \textsc{SlotMatch} on the YTVIS dataset using three random seeds (42, 101, 2048). The last row reports the average performance across seeds. Consistent results across runs indicate stable training.}
\label{table:random_seed}
\end{table}

\noindent \textbf{Random seeds and repeated runs.}
To evaluate the robustness of our method, all reported results are averaged across three runs using the following seeds: 42, 101 and 2048. Reported results in Table \ref{table:random_seed} reflect the average performance across these runs. In preliminary experiments, we found the standard deviations for FG-ARI and mBO across seeds to be within $\pm$0.06 and $\pm$0.29 on YTVIS, indicating stable convergence behavior. All random seeds were set using PyTorch Lightning's \texttt{seed\_everything} function, as well as each independent module's respective seed function, to ensure full reproducibility.

\noindent \textbf{Data access and preprocessing.}
We used public benchmark datasets: MOVi-E \cite{ghorbani2021movi}, YTVIS-2021 \cite{vis2021}, DAVIS \cite{Perazzi-CVPR-2016} and OVIS \cite{qi2022occluded}. All datasets are publicly available and can be downloaded from the respective official repositories. Preprocessing for MOVi-E follows the original 336$\times$336 center crop and normalization to $[-1, 1]$. For YTVIS, frames are resized using short-side resizing to 518 pixels with central cropping. OVIS is used for evaluation only. Temporal chunks of 4 frames are sampled per video for both training and evaluation. The corresponding preprocessing and data-related details are present in the configuration files. 

\noindent \textbf{Pre-trained models and licensing.}
We will release pre-trained weights for both teacher and student models on all datasets in our code repository\footnote{\url{https://github.com/dianagrigore/SlotMatch}}. Checkpoints will be provided with instructions for loading and evaluation. The code is released under the 
CC-BY-NC 4.0 license, and all dependencies are listed in the \texttt{environment.yml} file. The repository includes inference scripts, training pipelines, and evaluation tools for FG-ARI and mBO, along with an example configuration to reproduce results.

\end{document}